\documentclass[journal]{IEEEtran}

\usepackage{cite}
\usepackage{amsmath}
\usepackage{amssymb}
\usepackage{balance}
\usepackage{graphicx}
\usepackage{subcaption}
\usepackage{epstopdf}
\usepackage{algpseudocode}
\usepackage[usenames]{color}
\usepackage{algorithmicx,algorithm}
\usepackage{amsthm}
\usepackage{xcolor}

\DeclareMathOperator*{\argmax}{arg\,max}
\DeclareMathOperator*{\argmin}{arg\,min}
\DeclareMathOperator*{\arginf}{arg\,inf}
\theoremstyle{plain}
\newtheorem{lemma}{Lemma}
\newtheorem{theorem}{Theorem}

\theoremstyle{definition}
\newtheorem{definition}{Definition}
\newtheorem{assumption}{Assumption}
\theoremstyle{remark}
\newtheorem{remark}{Remark}

\begin{document}

\title{ByRDiE: Byzantine-resilient Distributed Coordinate Descent for Decentralized Learning}

\author{Zhixiong Yang,~\IEEEmembership{Graduate~Student~Member,~IEEE}, and Waheed U. Bajwa,~\IEEEmembership{Senior~Member,~IEEE}
\thanks{Some of the results reported in this paper were presented at the 2018 IEEE Data Science Workshop (DSW'18)~\cite{YangBajwa.ConfDSW18}. This work is supported in part by the NSF under award CCF-1453073, by the ARO under award W911NF-17-1-0546, and by the DARPA Lagrange Program under ONR/SPAWAR contract N660011824020.}
\thanks{The authors are with the Department of Electrical and Computer Engineering, Rutgers University--New Brunswick, 94 Brett Rd, Piscataway, NJ 08854, USA (Emails: {\tt\{zhixiong.yang,waheed.bajwa\}@rutgers.edu}).}
}

\maketitle

\begin{abstract}
Distributed machine learning algorithms enable learning of models from datasets that are distributed over a network without gathering the data at a centralized location. While efficient distributed algorithms have been developed under the assumption of faultless networks, failures that can render these algorithms nonfunctional occur frequently in the real world. This paper focuses on the problem of Byzantine failures, which are the hardest to safeguard against in distributed algorithms. While Byzantine fault tolerance has a rich history, existing work does not translate into efficient and practical algorithms for high-dimensional learning in \emph{fully} distributed (also known as decentralized) settings. In this paper, an algorithm termed \emph{Byzantine-resilient distributed coordinate descent} (ByRDiE) is developed and analyzed that enables distributed learning in the presence of Byzantine failures. Theoretical analysis (convex settings) and numerical experiments (convex and nonconvex settings) highlight its usefulness for high-dimensional distributed learning in the presence of Byzantine failures.
\end{abstract}

\begin{IEEEkeywords}
Byzantine failure, consensus, coordinate descent, decentralized learning, distributed optimization, empirical risk minimization, machine learning
\end{IEEEkeywords}

\section{Introduction}\label{section introduction}
One of the fundamental goals in machine learning is to learn a model that minimizes the statistical risk. This is typically accomplished through stochastic optimization techniques, with the underlying principle referred to as \emph{empirical risk minimization} (ERM)~\cite{Vapnik.ConfNIPS92,Vapnik2013nature,mohri2012foundations}. The ERM principle involves the use of a training dataset and tools from optimization theory. Traditionally, training data have been assumed available at a centralized location. Many recent applications of machine learning, however, involve the use of a dataset that is either distributed across different locations (e.g., the \emph{Internet of Things}) or that cannot be processed at a single machine due to its size (e.g., social network data). The ERM framework in this setting of distributed training data is often referred to as \emph{decentralized} or \emph{distributed} learning~\cite{predd2006distributed,boyd2011distributed}.

While there exist excellent works that solve the problem of distributed learning, all these works make a simplified assumption that all nodes in the network function as expected. Unfortunately, this assumption does not always hold true in practice; examples include cyber attacks, malfunctioning equipments and undetected failures~\cite{Driscoll2003byzantine,driscoll2004the}. When a node arbitrarily deviates from its intended behavior, it is termed to have undergone Byzantine failure~\cite{lamport1982byzantine}. While Byzantine failures are hard to detect in general, they can easily jeopardize the operation of the whole network~\cite{fischer1985impossibility,dutta2005best,sousa2012byzantine}.

In particular, with just a simple strategy, one can show that a single Byzantine node in the network can lead to failures of most state-of-the-art distributed learning algorithms~\cite{su2015fault,su2016fault,yang2016rdsvm}. The main contribution of this paper is to introduce and analyze an algorithm that completes the distributed learning task in the presence of Byzantine failures in the network.

\subsection{Related Works}
To achieve the goal of distributed learning, one usually sets up a distributed optimization problem by defining and minimizing a (regularized) loss function on the training data of each node. The resulting problem can then be solved by distributed optimization algorithms. Several types of distributed algorithms have been introduced in the past~\cite{Nedic2009Distributed,Ram2010distributed,forero2010consensus,mota2013admm,shi2014on,nedic2015distributed,Mokhtari2016decentralized,RajaBajwa.ITSP16,Mokhtari2017network}. The most common of these are gradient-based methods~\cite{Nedic2009Distributed, Ram2010distributed, nedic2015distributed}, which usually have low local computational complexity. Another type includes augmented Lagrangian-based methods~\cite{mota2013admm,shi2014on,forero2010consensus}, which iteratively update the primal and dual variables. These methods often require the ability to locally solve an optimization subproblem at each node. A third type includes second-order distributed methods~\cite{Mokhtari2016decentralized,Mokhtari2017network}, which tend to have high computational and/or communications costs. While any of these distributed optimization algorithms can be used for distributed learning, they all make the assumption that there are no failures in the network.

Byzantine-resilient algorithms for a variety of problems have been studied extensively over the years~\cite{lamport1982byzantine,minsky2003tolerating,rawat2011collaborative,vempaty2013distributed,ChenACC18}.
Byzantine-resilient algorithms for scalar- and vector-averaging distributed consensus were studied in \cite{vaidya2012iterative,vaidya2012matrix,Leblanc2013resilient,vaidya2013byzantine,vaidya2014iterative}. The algorithms proposed in \cite{su2015fault,su2016fault} extend some of these works from scalar consensus to scalar-valued distributed optimization, but they cannot be used in a straightforward manner for vector-valued distributed optimization problems. In the context of distributed learning, \cite{yang2016rdsvm} introduces a method to implement distributed support vector machines (SVMs) under Byzantine failures, but it neither provides theoretical guarantees nor generalizes to other learning problems. A number of recent works have also investigated Byzantine-resilient distributed learning in networks that are equipped with a central processor (often referred to as a \emph{paramater server})~\cite{blanchard2017byzantine,chen2017distributed,blanchard2017machine,damaskinos2018asynchronous,mhamdi2018hidden,xie2018generalized,draco2018chen,yin2018byzantine,alistarh2018byzantine,su2018securing,yin2018defending}. Theoretical guarantees developed in such works make extensive use of the fact that the parameter server is connected to all network nodes; as such, these guarantees cannot be generalized to Byzantine-resilient learning in \emph{fully} distributed networks.

\subsection{Our Contributions}
We have already noted several limitations of works like \cite{vaidya2012matrix,vaidya2013byzantine,Leblanc2013resilient,su2015fault,su2016fault,yang2016rdsvm,blanchard2017byzantine,chen2017distributed,blanchard2017machine,damaskinos2018asynchronous,mhamdi2018hidden,xie2018generalized,draco2018chen,yin2018byzantine,alistarh2018byzantine,su2018securing,yin2018defending,vaidya2012iterative,vaidya2014iterative} within the context of Byzantine-resilient learning in fully distributed settings. Additionally, one of the limitations of existing Byzantine-resilient algorithms such as~\cite{vaidya2014iterative} is that, when required to work with vector-valued data, they make a strong assumption on the network topology. Specifically, the smallest size of neighborhood of each node in the vector setting depends linearly on the dimensionality of the problem. This is impractical for most learning problems since the dimensionality of the training samples is usually much larger than the size of the neighborhood of each node. Within the specific context of Byzantine-resilient distributed machine learning, the main limitation of fully distributed algorithms proposed in~\cite{su2015fault,su2016fault} is that they only yield the minimizer of a convex combination of local empirical risk functions for scalar-valued problems. Since this (scalar) minimizer is usually different from the minimizer of the exact average of local loss functions, these works cannot guarantee by themselves alone that the outputs of any vector-valued algorithms that leverage similar ideas will converge to either the minimum empirical risk or the minimum statistical risk.

In contrast to prior works, our work has two main contributions. First, we propose a Byzantine-resilient algorithm that scales well with the dimensionality of distributed learning problems. The proposed algorithm is an \emph{inexact} distributed variant of coordinate descent~\cite{wright2015coordinate} that first breaks vector-valued distributed learning problems into a (possibly infinite) sequence of scalar-valued distributed subproblems and then solves these subproblems in a Byzantine-resilient manner by leveraging ideas from works such as \cite{su2015fault,su2016fault}. The inexactness here stems from the fact that---even when using an exact line search procedure---the subproblems cannot be solved exactly in the presence of Byzantine failures (cf.~Sec.~\ref{section scalar case}). This inexactness in the solution of each subproblem, whether or not an exact line search procedure is used, is one of the reasons the analytical techniques of prior works such as \cite{su2015fault,su2016fault} do not lead to theoretical guarantees for the proposed coordinate descent algorithm. In contrast, under the assumption of independent and identically distributed training samples across the network, we provide theoretical guarantees that the output of the proposed algorithm will lead to the minimum statistical risk with high probability. Our theoretical analysis, which forms the second main contribution of this work, also highlights the fact that the output of our algorithm can statistically converge to the minimizer of the statistical risk faster than using only local information by a factor that will be shown in the sequel.

\subsection{Notation and Organization}
Given a vector $v$ and a matrix $A$, $[v]_k$ and $[A]_{ij}$ denote their $k$-th and $(i,j)$-th element, respectively. We use $v\vert_{[v]_k=a'}$ to denote the vector formed by replacing the $k$-th element of $v$ with $a' \in \mathbb{R}$, while $(\cdot)^T$ denotes the transpose operation. We use $\|v\|$ and $\|v\|_\infty$ to denote the $\ell_2$- and $\ell_\infty$-norms of $v$, respectively, $\mathbf{1}$ to denote the vector of all ones, and $I$ to denote the identity matrix. Given a vector $v$ and a constant $\gamma > 0$, we use $B(v,\gamma) := \{v^\prime: \|v-v^\prime\| \leq \gamma\}$ to denote an $\ell_2$-ball of radius $\gamma$ centered around $v$. Given a set, $\vert \cdot\vert$ denotes its cardinality. We use the scaling relation $f(n) = \mathcal{O}(g(n))$ if $\exists c_o > 0, n_o : \forall n \geq n_o, f(n) \leq c_o g(n)$. Finally, unless stated otherwise, we use $\nabla f(w,(x,y))$ to denote the gradient of a function $f(w,(x,y))$ with respect to $w$.

The rest of this paper is organized as follows. Section~\ref{section problem formulation} gives the problem formulation. Section~\ref{section scalar case} discusses the proposed algorithm along with theoretical guarantees for consensus as well as statistical and algorithmic convergence. Numerical results corresponding to distributed convex and nonconvex learning on two different datasets are provided in Section~\ref{section numerical}, while Section~\ref{section conclusion} concludes the paper.

\section{Problem Formulation}\label{section problem formulation}
Given a connected network in which nodes have access to local training data, our goal is to learn a machine learning model from the distributed data, even in the presence of Byzantine failures. We begin with a model of our basic problem, which will be followed with a model for Byzantine failures and the final problem statement.

Consider a connected network of $M$ nodes, expressed as a directed, static graph $\mathcal{G}(J,\mathcal{E})$.
Here, the set $J:=\{1,\dots,M\}$ represents nodes in the network, while the set of edges $\mathcal{E}$ represents communication links between different nodes. Specifically, $(j,i) \in \mathcal{E}$ if and only if node $i$ can receive information from node $j$. Each node $j$ has access only to a local training set $S_j=\lbrace (x_{jn},y_{jn})\rbrace_{n=1}^{\vert S_j\vert}$. Let $x\in\mathbb{R}^P$ represent the training features satisfying $\Vert x\Vert\leq B$ for some constant $B$ and $y$ be the corresponding label. For classification, $y\in\lbrace -1,1\rbrace$, while $y\in\mathbb{R}$ for regression. We assume that the training samples are independent and identically distributed (i.i.d.) and drawn from an unknown distribution $D$, i.e., $(x_{jn},y_{jn})\sim D$. For simplicity, we assume that the cardinalities of the local training sets are the same, i.e., $\vert S_j\vert=N$. The generalization to the case when $S_j$'s are not equal is trivial.

\begin{remark}
Note that while the main problem is being formulated here under the supervised setting, our proposed framework and the final results are equally applicable under both unsupervised and semi-supervised settings.
\end{remark}

In machine learning, one would ideally like to collect all the data into one set $S=\lbrace (x_n,y_n)\rbrace_{n=1}^{\vert S\vert}$ and perform centralized training on $S$. The goal in that case is to learn a function that reliably maps $x$ to $y$. One popular mapping is $y=w^Tx$ (sometimes this is defined as $y=w^Tx+b$, which can be transformed into $y=w^Tx$ by adding one more dimension to $x$). To find a ``good" $w$, one first defines a (non-negative) loss function $\ell(w,(x,y))$, where the value of loss function increases when the difference between the mapping of $x$ and $y$ increases. To avoid overfitting, a (non-negative) regularizer $R(w)$ is often added to the loss function. Then one can solve for $w$ by statistically minimizing a regularized loss function
\begin{align}
  f(w,(x,y)) := \ell(w,(x,y)) + R(w).
\end{align}
The regularized $f(w,(x,y))$ is often referred to as \emph{risk function}. In this paper, we focus on the class of convex differentiable loss functions and strictly convex and smooth regularizers.\footnote{A function $R(w)$ is strictly convex if it satisfies $R(zw_1+(1-z)w_2) < z R(w_1)+(1-z)R(w_2)$ for any $z \in (0,1)$. Further, $R(w)$ is smooth if it is differentiable for all orders.} Examples include square loss $(1-y\cdot w^Tx)^2$, square hinge loss $\max(0,1-y\cdot w^Tx)^2$, logistic loss $\ln(1+e^{-y\cdot w^Tx})$ and $R(w)=\frac{\lambda}{2}\Vert w\Vert^2$. We also assume the gradient of the loss function $\ell(w,(x,y))$ is $L$-Lipschitz continuous. Since $R(w)$ is smooth, we formally state the Lipschitz assumption as follows.
\begin{assumption}\label{assumption lipschitz}
The risk function $f=\ell(w,(x,y))+R(w)$ satisfies $$\forall w_1, w_2, \ \Vert \nabla f(w_1,(x,y))-\nabla f(w_2,(x,y))\Vert\leq L\Vert w_1-w_2\Vert.$$
\end{assumption}
\noindent Note that Assumption~\ref{assumption lipschitz} implies the risk function itself is also Lipschitz continuous~\cite{sohrab2003basic}:
\begin{align*}
  \forall w_1, w_2, \ \Vert f(w_1,(x,y))-f(w_2,(x,y))\Vert\leq L'\Vert w_1-w_2\Vert.
\end{align*}

Centralized machine learning focuses on learning the ``best'' mapping from $x$ to $y$ in terms of the following stochastic optimization problem (referred to as \emph{statistical risk minimization}):
\begin{align}\label{eqn:stochastic.problem}
  w^\ast := \argmin_{w \in \mathbb{R}^p} \mathbb{E}[f(w,(x,y))],
\end{align}
where the expectation is with respect to the unknown distribution $D$. In this work, we assume without loss of generality that $\|w^\ast\|_\infty \leq \gamma^\ast$ for some constant $\gamma^\ast \in \mathbb{R}_+$. Owing to the equivalence of norms in finite-dimensional spaces, extensions of our analysis to bounds on $w^\ast$ in norms other than the $\ell_\infty$-norm are straightforward, but would not be pursued in this work. In the following, we will use $W$ to denote the smallest $\ell_\infty$-ball centered around the origin that contains $w^*$ and the \emph{sublevel set} of the \emph{statistical risk} $\mathbb{E}[f(w,(x,y))]$ at level $c_0 := \mathbb{E}[f(0,(x,y))]$.\footnote{Our choice of the level $c_0$ is tied to our initialization of the proposed algorithm, which is currently taken to be $w = 0$. This initialization, however, is without loss of generality and can be changed according to one's preference, resulting in a minor change in the definition of the set $W$.} Specifically, defining $W_\gamma := \{w \in \mathbb{R}^p: \|w\|_\infty \leq \gamma\}$ and
\begin{align*}
  \Gamma := \arginf_\gamma \gamma \ \! \ \text{s.t.} \ \! \ W_\gamma \supset \{w: \mathbb{E}[f(w,(x,y))] \leq c_0\},
\end{align*}
we have that $W := \{w \in \mathbb{R}^p: \|w\|_\infty \leq \Gamma\}$. We now make another assumption.
\begin{assumption}\label{assumption finite value}
For any $w\in W$, the loss function $\ell(w,(x,y))$ is bounded almost surely over all training samples, i.e., $$\forall (x,y)\in\bigcup_{j\in J} S_j, \quad 0 \leq \ell(w,(x,y))\leq C<\infty.$$
\end{assumption}
\noindent Note that Assumption~\ref{assumption finite value} would be satisfied for datasets with finite-valued training samples because of the Lipschitz continuity of $\ell(w,(x,y))$ and the compactness of $W$.
\begin{remark}
While the set $W$ has no algorithmic significance, it is later shown that the iterates of the proposed algorithm stay within it. Because of this reason, $W$ makes an appearance in our theoretical guarantees concerning statistical convergence.
\end{remark}

Since $D$ is unknown in~\eqref{eqn:stochastic.problem}, one cannot solve for $w^\ast$ directly. A broadly adopted alternative then is to utilize the training data and minimize the \emph{empirical risk} $\widehat{f}(w,S)$ in lieu of the statistical risk $\mathbb{E}[f(w,(x,y))]$, where
\begin{align}\label{eqn:empirical.risk}
  \widehat{f}(w,S) &:= \frac{1}{\vert S\vert}\sum_{n=1}^{\vert S\vert}f(w,S)\nonumber \\
    &:=\frac{1}{\vert S\vert}\sum_{n=1}^{\vert S\vert}\ell(w,(x_n,y_n)) + R(w).
\end{align}
In particular, the minimizer of the empirical risk $\widehat{f}(w,S)$ can be shown to converge to $w^\ast$ with high probability~\cite{Vapnik2013nature,shalev2009stochastic}. In the case of i.i.d.~training samples, and for a fixed probability of failure and ignoring some $\log$ terms, the rate of this \emph{statistical convergence} is known to be $\mathcal{O}\left(1/\sqrt{|S|}\right) \equiv \mathcal{O}\left(1/\sqrt{MN}\right)$ under mild assumptions on the problem~\cite{shalev2009stochastic}.

In many distributed learning problems, training data cannot be made available at a single location. This then requires learning $w^\ast$ in a distributed fashion, which can be done by employing distributed optimization techniques. The main idea of distributed optimization-based learning is to minimize the average of local empirical risks, i.e.,
\begin{align*}
  \frac{1}{M}\sum_{j=1}^M \widehat{f}(w,S_j)=\frac{1}{MN}\sum_{j=1}^M\sum_{n=1}^N\ell (w,(x_{jn},y_{jn})) + R(w).
\end{align*}
To achieve this goal, we need nodes to cooperate with each other by communicating over network edges. Specifically, define the neighborhood of node $i$ as $\mathcal{N}_i:=\lbrace j\in J:(j,i)\in \mathcal{E}\rbrace$. We say that node $j$ is a neighbor of node $i$ if $j\in \mathcal{N}_i$. Distributed learning algorithms proceed iteratively. In each iteration $(r+1)$ of the algorithm, node $j$ is expected to accomplish two tasks:
\begin{enumerate}
\item Update a local variable $w_j^r$ according to some (deterministic or stochastic) rule $g_j(\cdot)$, and
\item Broadcast the updated local variable to other nodes, where node $i$ can receive the broadcasted information from node $j$ \emph{only} if $j\in \mathcal{N}_i$.
\end{enumerate}

\subsection{Byzantine Failure Model}
While distributed learning via message passing is well understood~\cite{forero2010consensus,duchi2012dual}, existing protocols require all nodes to operate as intended. In contrast, the main assumption in this paper is that some of the nodes in the network can undergo Byzantine failures, formally defined as follows.
\begin{definition}
A node $j \in J$ is said to be Byzantine if during any iteration, it either updates its local variable using an update function $g_j^\prime(\cdot)\neq g_j(\cdot)$ or it broadcasts some value other than the intended update to its neighbors.
\end{definition}

In this paper, we assume there are $b$ Byzantine nodes in the network. Knowing the exact value of $b$, however, is not necessary. One can, for example, set $b$ to be an upper bound on the number of Byzantine nodes. Let $J' \subset J$ denote the set of nonfaulty nodes. Without loss of generality, we assume the nonfaulty nodes are labeled from 1 to $\vert J'\vert$. We now provide some definitions and an assumption that are common in the literature; see, e.g., \cite{su2015fault,su2016fault}.
\begin{definition}
A subgraph $\mathcal{G}_r$ of $\mathcal{G}$ is called a reduced graph if it is generated by ($i$) removing all Byzantine nodes along with their incoming and outgoing edges, and ($ii$) removing additionally up to $b$ incoming edges from each nonfaulty node.
\end{definition}
\begin{definition}
A ``source component'' of a graph is a collection of nodes such that each node in the source component has a directed path to every other node in the graph.
\end{definition}

\begin{assumption}\label{assumption reduced graph}
All reduced graphs $\mathcal{G}_r$ generated from $\mathcal{G}(J,\mathcal{E})$ contain a source component of cardinality at least $(b+1)$.
\end{assumption}
\noindent The purpose of Assumption~\ref{assumption reduced graph} is to ensure there is enough redundancy in the graph to tolerate Byzantine failures. Note that the total number of different reduced graphs one can generate from $\mathcal{G}$ is finite as long as $M$ is finite. So, in theory, Assumption~\ref{assumption reduced graph} can be certified for any graph. But efficient certification of this assumption remains an open problem. In the case of Erd\H{o}s--R\'{e}nyi graphs used in our experiments, however, we have observed that Assumption~\ref{assumption reduced graph} is typically satisfied whenever the ratio of the average incoming degree of the graph and the number of Byzantine nodes is high enough.

\begin{remark}\label{rem:choice.b}
Assumption~\ref{assumption reduced graph} has been stated to guarantee resilience against the \emph{worst-case} attack scenario in which all Byzantine nodes concentrate in the neighborhood of any one of the network nodes. While such worst-case analysis is customary in the literature on Byzantine fault tolerance~\cite{su2016fault,Leblanc2013resilient,vaidya2013byzantine}, it does impose stringent constraints on the network topology. Such topology constraints, however, seem to be unavoidable for the types of ``local screening'' algorithms being considered in the fully distributed setting of this paper.
\end{remark}

\subsection{Problem Statement}
Our goal is to develop a Byzantine fault-tolerant algorithm for distributed learning under Assumptions~\ref{assumption lipschitz}--\ref{assumption reduced graph}. Specifically, under the assumption of at most $b$ Byzantine nodes in the network, we need to accomplish the following:
\begin{enumerate}
  \item Achieve consensus among nonfaulty nodes in the network, i.e., $w_j^r=w_i^r$ $\forall i,j\in J'$ as the number of algorithmic iterations $r\rightarrow\infty$; and
  \item Guarantee $w_j^r \rightarrow w^\ast$ $\forall j\in J'$ as the number of training samples at nonfaulty nodes $N\rightarrow\infty$.
\end{enumerate}
In particular, the latter objective requires understanding both the statistical convergence and the algorithmic convergence of the distributed empirical risk minimization problem in the presence of Byzantine failures.

\section{Byzantine-resilient Distributed Coordinate Descent for Decentralized Learning}\label{section scalar case}
In distributed learning, one would ideally like to solve the empirical risk minimization (ERM) problem
\begin{equation}\label{opt cant be found}
w_{opt} =\argmin\limits_{w\in \mathbb{R}^P} \frac{1}{MN}\sum\limits_{j=1}^M\sum\limits_{n=1}^N\ell (w,(x_{jn},y_{jn})) + R(w)
\end{equation}
at each node $j\in J'$ and show that $w_{opt}\rightarrow w^\ast$ as $N\rightarrow\infty$. However, we know from prior work~\cite{su2015byzantine} that this is infeasible when $b\geq 1$. Nonetheless, we establish in the following that distributed strategies based on coordinate descent algorithms~\cite{wright2015coordinate} can still be used to solve a variant of~\eqref{opt cant be found} at nonfaulty nodes and guarantee that the solutions converge to the minimizer $w^\ast$ of the statistical risk for the case of i.i.d. training data. We refer to our general approach as \textbf{By}zantine-\textbf{R}esilient  \textbf{Di}stributed coordinate d\textbf{E}scent (ByRDiE), which is based on key insights gleaned from two separate lines of prior work. First, it is known that certain types of scalar-valued distributed optimization problems can be inexactly solved in the presence of Byzantine failures~\cite{su2015fault,su2016fault}. Second, coordinate descent algorithms break down vector-valued optimization problems into a sequence of scalar-valued problems~\cite{wright2015coordinate}. The algorithmic aspects of ByRDiE leverage these insights, while the theoretical aspects leverage tools from Byzantine-resilient distributed consensus~\cite{vaidya2012matrix}, optimization theory~\cite{wright2015coordinate}, stochastic convex optimization~\cite{shalev2009stochastic}, and statistical learning theory~\cite{Vapnik2013nature}.

\subsection{Algorithmic Details}
ByRDiE involves splitting the ERM problem~\eqref{opt cant be found} into $P$ one-dimensional subproblems using coordinate descent and then solving each scalar-valued subproblem using the Byzantine-resilient approach described in~\cite{su2015fault}. The exact implementation is detailed in Algorithm~\ref{coordinate descent algorithm}. The algorithm can be broken into an outer loop (Step~\ref{outer loop}) and an inner loop (Step~\ref{inner loop}). The outer loop is the coordinate descent loop, which breaks the vector-valued optimization problem in each iteration $r$ into $P$ scalar-valued subproblems. The inner loop solves a scalar-valued optimization problem in each iteration $t$ and ensures resilience to Byzantine failures. We assume the total number of iterations $\bar{r}$ for coordinate descent are specified during initialization. We use $[w_j^r(t)]_k$ to denote the $k$-th element of $w_j$ at the $r$-th iteration of the coordinate descent loop and the $t$-th iteration of the $k$-th subproblem (coordinate). Without loss of generality, we initialize $[w_j^1(1)]_k=0, \forall k=1,\dots,P$.

\begin{algorithm}[t]
\caption{Byzantine-resilient distributed coordinate descent} \label{coordinate descent algorithm}
\begin{algorithmic}[1]
\Require $S_1, S_2,\dots, S_M, \lbrace\rho(\tau)\rbrace_{\tau=1}^{\infty}, b \in \mathbb{N}, \bar{r} \in \mathbb{N}, T \in \mathbb{N}$
\State \textbf{Initialize:} $r \gets 1$, $t \gets 1$, and $\forall j \in J', w_j^r(t) \leftarrow 0$
\For{$r=1,2,\dots,\bar{r}$}  \label{outer loop}
\For{$k=1,2,\dots,P$}\label{k loop}
\For{$t=1,2,\dots,T$}\label{inner loop}
\For{$j=1,2,\dots,\vert J'\vert$} \textbf {(in parallel)}
\State Receive $[w_i^r(t)]_k$ from all $i \in \mathcal{N}_j$\label{broadcast}
\State Find $\mathcal{N}_j^s(r,k,t)$, $\mathcal{N}_j^l(r,k,t)$, $\mathcal{N}_j^\ast(r,k,t)$ according to (\ref{define ns}), (\ref{define nl}), and (\ref{define nstar})\label{screen}
\State Update $[w_j^r(t+1)]_k$ as in (\ref{scalar update})\label{update}
\EndFor
\EndFor
\EndFor
\State $w_j^{r,T} \leftarrow w_j^{r}(T+1), \forall j\in J'$\label{alg:iteration.output}
\State $w_j^{r+1}(1) \leftarrow w_j^{r,T}, \forall j\in J'$
\EndFor
\Ensure{$\left\lbrace w^{\bar{r},T}_j\right\rbrace_{j\in J'}$}
\end{algorithmic}
\end{algorithm}

We now fix some $r$ and $k$, and focus on the implementation of the inner loop (Step~\ref{inner loop}). Every node has some $[w_j^r(1)]_k$ at the start of the inner loop ($t = 1$). During each iteration $t$ of this loop, all (nonfaulty) nodes engage in the following: \emph{broadcast}, \emph{screening}, and \emph{update}. In the broadcast step (Step~\ref{broadcast}), all nodes $i \in J$ broadcast $[w_i^r(t)]_k$'s and each node $j \in J$ receives $[w_i^r(t)]_k, \forall i\in \mathcal{N}_j$. During this step, a node can receive values from both nonfaulty \emph{and} Byzantine neighbors. The main idea of the screening step (Step~\ref{screen}) is to reject values at node $j$ that are either ``too large'' or ``too small'' so that the values being used for update by node $j$ in each iteration will be upper and lower bounded by a set of values generated by nonfaulty nodes. To this end, we partition $\mathcal{N}_j$ into 3 subsets $\mathcal{N}_j^\ast(r,k,t)$, $\mathcal{N}_j^s(r,k,t)$ and $\mathcal{N}_j^l(r,k,t)$, which are defined as following:
\begin{align}\label{define ns}
\mathcal{N}_j^s(r,k,t)&=\argmin\limits_{X: X\subset \mathcal{N}_j, \vert X\vert =b}\sum\limits_{i\in X}[w_i^r(t)]_k,\\
\label{define nl}
\mathcal{N}_j^l(r,k,t)&=\argmax\limits_{X: X\subset \mathcal{N}_j, \vert X\vert =b}\sum\limits_{i\in X}[w_i^r(t)]_k,\\
\label{define nstar}
\mathcal{N}_j^\ast(r,k,t)&=\mathcal{N}_j\setminus\mathcal{N}_j^s(r,k,t)\setminus\mathcal{N}_j^l(r,k,t).
\end{align}
The step is called screening because node $j$ only uses $[w_i^r(t)]_k$'s from $\mathcal{N}^\ast_j(r,k,t)$ to update its local variable. Note that there might still be $[w_i^r(t)]_k$'s received from Byzantine nodes in $\mathcal{N}^\ast_j(r,k,t)$. We will see later, however, that this does not effect the workings of the overall algorithm.

The final step of the inner loop in ByRDiE is the update step (Step~\ref{update}). Using $[\nabla \widehat{f}(w_j^r(t), S_j)]_k$ to denote the $k$-th element of $\nabla \widehat{f}(w_j^r(t), S_j)$, we can write this update step as follows:
\begin{align}\label{scalar update}
[w_j^r(t+1)]_k &=\frac{1}{\vert\mathcal{N}_j\vert -2b+1}\sum\limits_{i\in \mathcal{N}^\ast_j(r,k,t)\cup \lbrace j\rbrace}[w_i^r(t)]_k\nonumber\\
&\qquad -\rho(r+t-1)[\nabla \widehat{f}(w_j^r(t), S_j)]_k,
\end{align}
where $\lbrace \rho (\tau)\rbrace_{\tau=1}^\infty$ are square-summable (but not summable), diminishing stepsizes: $0<\rho(\tau+1)\leq\rho(\tau)$, $\sum_\tau \rho(\tau)=\infty$, and $\sum_\tau \rho^2(\tau)<\infty$. Notice that $[w^r_j(T+1)]_k$ is updated after the $k$-th subproblem of coordinate descent in iteration $r$ finishes and it stays fixed until the start of the $k$-th subproblem in the $(r+1)$-th iteration of coordinate descent. An iteration $r$ of the coordinate descent loop is considered complete once all $P$ subproblems within the loop are solved. The local variable at each node $j$ at the end of this iteration is then denoted by $w^{r,T}_j$ (Step~\ref{alg:iteration.output}). We also express the output of the whole algorithm as $\lbrace w^{\bar{r},T}_j\rbrace_{j\in J'}$. Finally, note that while Algorithm~\ref{coordinate descent algorithm} cycles through the $P$ coordinates of the optimization variables in each iteration $r$ in the natural order, one can use any permutation of $\{1,\dots,P\}$ in place of this order.

We conclude this discussion by noting that the parameter $T$ in ByRDiE, which can take any value between $1$ and $\infty$, trades off consensus among the nonfaulty nodes and the convergence rate; see Sec.~\ref{ssec:choice.T} for further discussion on this tradeoff and Sec.~\ref{section numerical} for numerical experiments that highlight this tradeoff. Our theoretical analysis of ByRDiE focuses on the two extreme cases of $T \to \infty$ and $T=1$. In both cases, we establish in the following that the output of ByRDiE at nonfaulty nodes converges in probability to the minimum of the statistical risk. Convergence guarantees for a finite-valued $T > 1$ can be obtained from straightforward modifications of the analytical techniques used in the following.

\subsection{Theoretical Guarantees: Consensus}\label{ssec:consensus}
We now turn our attention to theoretical guarantees for ByRDiE. We first show that it leads to consensus among nonfaulty nodes in the network, i.e., all nonfaulty nodes agree on the same variable, in the limit of large $\bar{r}$ and/or $T$. Then we show in the next sections that the output of ByRDiE converges to the statistical optimum in the limit of large $\bar{r}$ for the two extreme choices of $T$: $T\to\infty$ and $T=1$.

Let us begin by establishing the claim of consensus. To this end, focusing exclusively on dimension $k$ in ByRDiE, we see that the $k$-th dimension of each local variable is updated $\bar{r}T$ times. Fixing any $k$, we can use an index $m := (r-1)T+t$ to denote the sequence generated for the $k$-th dimension up to iteration $(r,t)$ in ByRDiE. We define a vector $\Omega(m)\in\mathbb{R}^{\vert J'\vert}$ such that $[\Omega(m)]_j:=[w_j^r(t)]_k$ and let $\omega_j(m)$ denote $[\Omega(m)]_j$. Similarly, we define a vector $G(m)\in\mathbb{R}^{\vert J'\vert}$ such that $[G(m)]_j:=[\nabla \widehat{f}(w_j^r(t),S_j)]_k$. Next, let $\bar{\rho}(m) := \rho(r+t-1)$ and note that $\bar{\rho}(m)$ also satisfies $\bar{\rho}(m)\to 0$, $\sum_{m=1}^\infty\bar{\rho}(m)=\infty$ and $\sum_{m=1}^\infty\bar{\rho}^2(m)<\infty$. We can now express the update of the sequence corresponding to the $k$-th dimension at nonfaulty nodes in a matrix form as follows:
\begin{equation}\label{write update in matrix form}
    \Omega(m+1)=Y(m)\Omega(m)-\bar{\rho}(m)G(m),
\end{equation}
where $Y(m)$ is a matrix that is fully specified in the following.

Let $\mathcal{N}_j'$ and $\mathcal{N}_j^b$ denote the nonfaulty nodes and the Byzantine nodes, respectively, in the neighborhood of $j \in J'$, i.e., $\mathcal{N}_j'=J'\cap \mathcal{N}_j$ and $\mathcal{N}_j^b=\mathcal{N}_j\setminus\mathcal{N}_j'$. Notice that one of two cases can happen during each iteration at node $j \in J'$:
\refstepcounter{equation}
\begin{align}
  \mathcal{N}_j^\ast(m)\cap\mathcal{N}_j^b &\neq \emptyset, \quad \text{or} \tag{\theequation(a)}\label{eqn:consensus.case.a}\\
  \mathcal{N}_j^\ast(m)\cap\mathcal{N}_j^b &= \emptyset. \tag{\theequation(b)}\label{eqn:consensus.case.b}
\end{align}
For case \eqref{eqn:consensus.case.a}, since $\vert\mathcal{N}_j^b\vert\leq b$ and $\vert\mathcal{N}_j^s(m)\vert=\vert \mathcal{N}_j^l(m)\vert = b$, we must have $\mathcal{N}_j^s(m)\cap\mathcal{N}_j'\neq\emptyset$ and $\mathcal{N}_j^l(m)\cap\mathcal{N}_j'\neq\emptyset$. Then for each $i\in\mathcal{N}_j^\ast(m)\cap\mathcal{N}_j^b$, $\exists s_j^i\in \mathcal{N}_j^s(m)\cap\mathcal{N}_j'$ and $l_j^i\in \mathcal{N}_j^l(m)\cap\mathcal{N}_j'$ satisfying $\omega_{s_j^i}(m) \leq \omega_i(m) \leq \omega_{l_j^i}(m)$. Therefore, we have for each $i\in \mathcal{N}_j^\ast(m)\cap\mathcal{N}_j^b$ that $\exists \theta_i^j(m) \in [0,1]$ such that the following holds:
\begin{equation}
\omega_i(m)=\theta_i^j(m) \omega_{s_j^i}(m)+(1-\theta_i^j(m))\omega_{l_j^i}(m).
\end{equation}
We can now rewrite the update at node $j \in J'$ as follows:
\begin{align}\label{separate update wj}
  &\omega_j(m+1) = \frac{1}{\vert\mathcal{N}_j\vert-2b+1}\Big(\omega_j(m)+\!\!\!\!\!\!\!\!\sum\limits_{i\in\mathcal{N}'\cap\mathcal{N}_j^\ast(m)}\omega_i(m) \nonumber\\
  &+\sum\limits_{i\in\mathcal{N}^b\cap\mathcal{N}_j^\ast(m)}\big(\theta_i^j(m)\omega_{s_j^i}(m)+(1-\theta_i^j(m))\omega_{l_j^i}(m)]_k\big)\Big)\nonumber\\
  &\qquad\qquad+\bar{\rho}(m)[G(m)]_j.
\end{align}

It can be seen from \eqref{separate update wj} that the screening rule in ByRDiE effectively enables nonfaulty nodes to replace data received from Byzantine nodes with convex combinations of data received from nonfaulty nodes in their neighborhoods. This enables us to express the updates at nonfaulty nodes in the form~\eqref{write update in matrix form}, with the entries of $Y(m)$ given by
\begin{equation}\label{elements in M}
[Y(m)]_{ji}=\begin{cases}
    \frac{1}{\vert\mathcal{N}_j\vert-2b+1}, &i=j,\\
    \frac{1}{\vert\mathcal{N}_j\vert-2b+1}, & i\in\mathcal{N}'\cap\mathcal{N}_j^\ast(m),\\
    \sum\limits_{n\in\mathcal{N}^b\cap\mathcal{N}_j^\ast(m)}\frac{\theta_n^j(m)}{\vert\mathcal{N}_j\vert-2b+1}, &i\in\mathcal{N}'\cap\mathcal{N}_j^s(m),\\
    \sum\limits_{n\in\mathcal{N}^b\cap\mathcal{N}_j^\ast(m)}\frac{1-\theta_n^j(m)}{\vert\mathcal{N}_j\vert-2b+1}, &i\in\mathcal{N}'\cap\mathcal{N}_j^l(m),\\
    0,&\text{otherwise}.
\end{cases}
\end{equation}
Notice further that, since $\mathcal{N}_j^\ast(m)\cap\mathcal{N}_j^b=\emptyset$ (cf.~\eqref{separate update wj}), case \eqref{eqn:consensus.case.b} corresponds to a special case of case \eqref{eqn:consensus.case.a} in which we keep only the first, second, and last rows of (\ref{elements in M}). It is worth noting here that our forthcoming proof does not require knowledge of $Y(m)$; in particular, since the choices of $s_j^i$ and $l_j^i$ are generally not unique, $Y(m)$ itself is also generally not unique. The main thing that matters here is that $Y(m)$ will always be a row stochastic matrix; we refer the reader to \cite{vaidya2012matrix} for further properties of $Y(m)$.

In order to complete our claim that nonfaulty nodes achieve consensus under ByRDiE, even in the presence of Byzantine failures in the network, fix an arbitrary $m_0 \geq 0$ and consider $m > m_0$. It then follows from \eqref{write update in matrix form} that
\begin{align}\label{expression of wt+1 long}
    &\Omega(m+1)=Y(m)\Omega(m)-\bar{\rho}(m)G(m)\nonumber\\
        &=Y(m)Y(m-1)\cdots Y(m_0)\Omega(m_0) -\bar{\rho}(m)G(m)\nonumber\\
        &\qquad-\sum\limits_{\tau=m_0}^{m-1} Y(m)Y(m-1)\cdots Y(\tau+1)\bar{\rho}(\tau)G(\tau).
\end{align}
We now define matrices $$\Phi(m,m_0):=Y(m)Y(m-1)\cdots Y(m_0),$$ $\Phi(m,m):=Y(m)$, and $\Phi(m,m+1):=I$. Notice that $\Phi(m,m_0)$ is also row stochastic since it is a product of row-stochastic matrices. We can then express \eqref{expression of wt+1 long} as
\begin{equation}\label{expression of wt+1}
\Omega(m+1)=\Phi(m,m_0)\Omega(m_0)-\sum\limits_{\tau=m_0}^m\Phi(m,\tau+1)\bar{\rho}(\tau)G(\tau).
\end{equation}
Next, we need two key properties of $\Phi(m,m_0)$ from~\cite{su2015byzantine}.
\begin{lemma}[\!\cite{su2015byzantine}]\label{lemma property 1 of phi}
Suppose Assumption~\ref{assumption reduced graph} holds. Then for any $m_0 \geq 0$, there exists a stochastic vector $\pi(m_0)$ such that
\begin{equation}\label{property 1 of phi}
\lim\limits_{m\rightarrow\infty}\Phi(m,m_0)=\mathbf{1}\pi^T(m_0).
\end{equation}
\end{lemma}

In words, Lemma~\ref{lemma property 1 of phi} states that the product of the row-stochastic matrices $Y(m)$ converges to a steady-state matrix whose rows are identical and stochastic. We can also characterize the rate of this convergence; to this end, let $\psi$ denote the total number of reduced graphs that can be generated from $\mathcal{G}$, and define $\nu:=\psi\vert J'\vert$, $\mathcal{N}_{max} := \max\limits_{j\in J'}\vert\mathcal{N}_j\vert$, and $$\mu :=1-\frac{1}{(2\mathcal{N}_{max}-2b+1)^\nu}.$$
\begin{lemma}[\!\cite{su2015byzantine}]\label{lemma property 2 of phi}
Suppose Assumption~\ref{assumption reduced graph} holds. We then have that $\forall m_0 \geq 0$ ,
\begin{equation}\label{property 2 of phi}
\Big| [\Phi(m,m_0)]_{ji}-[\pi(m_0)]_i\Big| \leq\mu^{(\frac{m-m_0+1}{\nu})}.
\end{equation}
\end{lemma}
Lemma~\ref{lemma property 2 of phi} describes the rate at which the rows of $\Phi(m,m_0)$ converge to $\pi(m_0)$. We now leverage this result and show that the nonfaulty nodes achieve consensus under ByRDiE in the limit of large $m$, which translates into $\bar{r} \rightarrow \infty$ and/or $T \rightarrow \infty$. To this end, under the assumption of $m_0 = 1$, we have from \eqref{expression of wt+1} the following expression:
\begin{equation}\label{expression of wt+1 from w0}
\Omega(m+1)=\Phi(m,1)\Omega(1)-\sum\limits_{\tau=1}^m\Phi(m,\tau+1)\bar{\rho}(\tau)G(\tau).
\end{equation}
Next, suppose the nonfaulty nodes stop computing local gradients at time step $m$ and use $G(m+m')=0$ for $m' \geq 0$. Then, defining $V(m) := \lim\limits_{m'\rightarrow\infty}\Omega(m+m'+1)$, we obtain:
\begin{align}\label{define vt}
V(m) &= \lim\limits_{m'\rightarrow\infty}\Phi(m+m',1)\Omega(1)\nonumber \\
 &\qquad\qquad -\lim\limits_{m'\rightarrow\infty}\sum\limits_{\tau=1}^{m+m'}\Phi(m+m',\tau)\bar{\rho}(\tau)G(\tau)\nonumber \\
&=\mathbf{1}\pi^T(1)\Omega(1)-\sum\limits_{\tau=1}^{m-1}\mathbf{1}\pi^T(\tau)\bar{\rho}(\tau)G(\tau).
\end{align}
Notice from \eqref{define vt} that all elements in the vector $V(m) \in \mathbb{R}^{|J'|}$ are identical. Recall that $V(m)$ is obtained by looking at only one dimension $k$ of the optimization variable in ByRDiE; in the following, we use $v^k(m)$ to denote the identical elements of $V(m)$ corresponding to dimension $k$. We then have the following result concerning nonfaulty nodes.
\begin{theorem}[Consensus Behavior of ByRDiE]\label{lemma consensus}
Let Assumptions \ref{assumption lipschitz} and  \ref{assumption reduced graph} hold and fix $\bar{m} = \bar{r}T + 1$. Then,
\begin{align}
  \forall j\in J', \forall k \in \{1,\dots,P\}, \ \left[w_j^{\bar{r},T}\right]_k \to v^k(\bar{m})
\end{align}
as $\bar{r}\to\infty$ and/or $T\to\infty$.
\end{theorem}
Theorem~\ref{lemma consensus}, whose proof is given in Appendix~\ref{proof of lemma consensus}, establishes consensus at the nonfaulty nodes under ByRDiE when $\bar{r}\to\infty$ and/or $T\to\infty$. We conclude this discussion by also stating the rate at which the iterates of ByRDiE achieve consensus. To this end, we define the consensus vector $\bar{V}(r) \in \mathbb{R}^P$ in iteration $r$ as $[\bar{V}(r)]_k := v^k(r)$. To keep the notation simple, we limit ourselves to $T=1$ and use $w_j^{r}$ to denote $w_j^{r,1}$. Nonetheless, a similar result holds for other values of $T$.
\begin{theorem}[Consensus Rate for ByRDiE]\label{lemma:consensus.rate}
Let Assumptions \ref{assumption lipschitz} and  \ref{assumption reduced graph} hold. Then, fixing $T=1$, the iterates of ByRDiE satisfy:
\begin{align}
  \forall j \in J', \ \|w_j^{r} - \bar{V}(r)\| = \mathcal{O}\left(\sqrt{P} \rho(r)\right),
\end{align}
where $\bar{V}(r)$ denotes the consensus vector in iteration $r$.
\end{theorem}
Theorem~\ref{lemma:consensus.rate}, which is a straightforward consequence of the proof of Theorem~\ref{lemma consensus} (cf.~\eqref{consensus rate}), guarantees a sublinear rate for consensus; indeed, choosing the stepsize evolution to be $\rho(r) = \mathcal{O}\left(1/r\right)$ gives us $\|w_j^{r} - \bar{V}(r)\| = \mathcal{O}\left(\sqrt{P}/r\right)$.

\subsection{Theoretical Guarantees: Convergence for $T \to \infty$}
We now move to the second (and perhaps the most important) claim of this paper. This involves showing that the output of ByRDiE converges in probability to the minimizer (and minimum) of the statistical risk (cf.~\eqref{eqn:stochastic.problem}) for two extreme cases: \emph{Case~I:} $T\to\infty$ and \emph{Case~II:} $T=1$. We start our discussion with the case of $T \rightarrow \infty$, in which case an auxiliary lemma simply follows from~\cite[Theorem 2]{su2015fault} (also, see~\cite{vaidya2012matrix}).
\begin{lemma}\label{scalar convergence}
Let Assumptions~\ref{assumption lipschitz} and~\ref{assumption reduced graph} hold, and let the $k$-th subproblem of the coordinate descent loop in iteration $r$ of ByRDiE be initialized with some $\lbrace w_j\rbrace_{j\in J'}$. Then, as $T\to\infty$, $$\forall j \in J', \left[w_j^{r,T}\right]_k\to\argmin\limits_{w'\in \mathbb{R}}\sum_{j \in J'}\alpha_j(r,k)\widehat{f}(w_j\vert_{[w_j]_k=w'},S_j)$$
for some $\alpha_j(r,k)\geq 0$ such that $\sum_{j \in J'}\alpha_j(r,k)=1$.
\end{lemma}

Lemma~\ref{scalar convergence} shows that each subproblem of the coordinate descent loop in ByRDiE under Case~I converges to the minimizer of some convex combination of local empirical risk functions of the nonfaulty nodes with respect to each coordinate. In addition, Lemma~\ref{scalar convergence} guarantees that consensus is achieved among the nonfaulty nodes at the end of each coordinate descent loop under Case~I. Note that while this fact is already known from Theorems~\ref{lemma consensus} and \ref{lemma:consensus.rate}, Lemma~\ref{scalar convergence} helps characterize the consensus point. In summary, when nonfaulty nodes begin a coordinate descent subproblem with identical local estimates and $T\to\infty$, they are guaranteed to begin the next subproblem with identical local estimates.

We now fix $(r,k)$ and use $\tilde{w}_k^r$ to denote the identical initial local estimates at nonfaulty nodes at the beginning of $k$-th subproblem of the coordinate descent loop in the $r$-th iteration of ByRDiE under Case~I. Next, we define $h_k^r(w')$ and $H_k^r(w')$ for $w'\in \mathbb{R}$ as
\begin{align}
h_k^r(w') &:=\mathbb{E}[f(\tilde{w}_k^r\vert_{[\tilde{w}_k^r]_k=w'},(x,y))], \quad \text{and}\label{eqn:def.hkr}\\
H_k^r(w') &:=\sum_{j \in J'} \alpha_j(r,k)\widehat{f}(\tilde{w}_k^r\vert_{[\tilde{w}_k^r]_k=w'},S_j)\label{eqn:def.Hkr}
\end{align}
for some $\alpha_j(r,k)\geq 0$ such that $\sum_{j \in J'} \alpha_j(r,k)=1$. Note that $h_k^r(\cdot)$ is strictly convex and Lipschitz continuous. Now for fixed $r$ and $k$, define
\begin{align}
&w^\star :=\argmin\limits_{w'\in\mathbb{R}}h_k^r(w'), \quad \text{and}\label{define wastastk}\\
&\widehat{w} :=\argmin\limits_{w'\in\mathbb{R}}H_k^r(w')\label{define whatk}.
\end{align}
\begin{remark}
It should be evident to the reader from \eqref{eqn:def.hkr} and \eqref{define wastastk} that the univariate stochastic function $h_k^r(w')$ depends on $\tilde{w}_k^r$ and its (scalar-valued) minimizer $w^\star$, which should not be confused with the vector-valued statistical minimizer $w^*$ in \eqref{eqn:stochastic.problem}, is a function of $r$ and $k$. Similarly, it should be obvious from \eqref{eqn:def.Hkr} and \eqref{define whatk} that $H_k^r(w')$ depends on $\tilde{w}_k^r$ and $\{\alpha_j(r,k)\}_{j \in J'}$, while its minimizer $\widehat{w}$ is also a function of $r$ and $k$. We are dropping these explicit dependencies here for ease of notation.
\end{remark}
In words, if one were to solve the statistical risk minimization problem~\eqref{eqn:stochastic.problem} using (centralized) coordinate descent then $w^\star$ will be the $k$-th component of the output of coordinate descent after update of each coordinate $k$ in every iteration $r$. In contrast, $\widehat{w}$ is the $k$-th component of the outputs of ByRDiE after update of each coordinate $k$ in every iteration $r$ (cf.~Lemma~\ref{scalar convergence}). While there exist works that relate the empirical risk minimizers to the statistical risk minimizers (see, e.g.,~\cite{shalev2009stochastic}), such works are not directly applicable here because of the fact that $H_k^r(w')$ in this paper changes from one pair $(r,k)$ of indices to the next. Nonetheless, we can provide the following uniform statistical convergence result for ByRDiE under Case~I that relates the empirical minimizers $\{\widehat{w}\}$ to the statistical minimizers $\{w^\star\}$.

\begin{theorem}[Statistical Convergence Rate for ByRDiE]\label{lemma vector case high probability}
Let $P$ and $|J'|$ be fixed, $\bar{r}$ be any (arbitrarily large) positive integer, and $\{\alpha_j(r,k) \geq 0, j\in J'\}_{r,k=1}^{\bar{r},P}$ be any arbitrary collection satisfying $\sum_{j \in J'} \alpha_j(r,k)=1$. Let $|\widehat{w}| \leq \Gamma$, $|w^\star| \leq \Gamma$, and $\{\tilde{w}^r_k\}_{r,k} \subset W$, and define $\bar{a} := \max_{(r,k)} \sqrt{\sum_{j \in J'} \alpha_j^2(r,k)}$. Then, as long as Assumptions~\ref{assumption lipschitz} and \ref{assumption finite value} hold, we have $\forall \epsilon > 0$
\begin{align}\label{thm:statistical.convergence.byrdie.bound}
  \sup_{r,k} \left[h^r_k(\widehat{w}) - h^r_k(w^\star)\right] < \epsilon
\end{align}
with probability exceeding
\begin{align}\label{thm:statistical.convergence.byrdie.prob}
1 - 2\exp\left(-\frac{4|J'|N \epsilon^2}{c_1^2 |J'| \bar{a}^2 + \epsilon^2} + |J'|\log\left(\frac{c_2}{\epsilon}\right) + P\log\left(\frac{c_3}{\epsilon}\right)\right),
\end{align}
where $c_1 := 8C$, $c_2 := 24 C |J'|$, and $c_3 := 24 L' \Gamma P$.
\end{theorem}

The proof of this theorem is provided in Appendix~\ref{proof of lemma vector case high probability}. In words, ignoring minor technicalities that are resolved in Theorem~\ref{coordinate descent theorem} in the following, Theorem~\ref{lemma vector case high probability} states that the coordinate-wise outputs $\{\widehat{w}\}$ of ByRDiE for \emph{all} $(r,k)$ under Case~I achieve, with high probability, almost the same statistical risk as that obtained using the corresponding coordinate-wise statistical risk minimizers $\{w^\star\}$. We now leverage this result to prove that the iterates of ByRDiE at individual nodes achieve statistical risk that converges to the minimum statistical risk achieved by the statistical risk minimizer (vector) $w^*$ (cf.~\eqref{eqn:stochastic.problem}).

\begin{theorem}[Convergence Behavior of ByRDiE]\label{coordinate descent theorem}
Let Assumptions~\ref{assumption lipschitz}--\ref{assumption reduced graph} hold. Then, $\forall j \in J', \forall \epsilon > 0,$ and $T \to \infty,$ we have
\begin{align}\label{thm:byrdie.convergence.caseI}
  \lim_{\bar{r} \to \infty} \left[\mathbb{E}[f(w_j^{\bar{r},T},(x,y))] - \mathbb{E}[f(w^\ast,(x,y))]\right] < \epsilon
\end{align}
with probability exceeding
\begin{align}\label{thm:byrdie.convergence.caseI.prob}
1 - 2\exp\left(-\frac{4 |J'|N \epsilon^2}{{c_1'}^2 |J'| \bar{a}^2 + \epsilon^2} + |J'|\log\left(\frac{c_2'}{\epsilon}\right) + P\log\left(\frac{c_3'}{\epsilon}\right)\right),
\end{align}
where $c_1' := c_1 c_4$, $c_2' := c_2 c_4$, and $c_3' := c_3 c_4$ for $c_4 := 2PL\Gamma$, and $(\bar{a}, c_1, c_2, c_3)$ are as defined in Theorem~\ref{lemma vector case high probability}.
\end{theorem}

The proof of this theorem is given in Appendix~\ref{proof.byrdie.convergence.caseI}. We now make a couple of remarks concerning Theorem~\ref{coordinate descent theorem}. First, note that the uniqueness of the minimum of strictly convex functions coupled with the statement of Theorem~\ref{coordinate descent theorem} guarantee that $\forall j \in J', w^{\bar{r},T}_j \to w^\ast$ with high probability.

Second, Theorem~\ref{coordinate descent theorem} helps crystallize the advantages of distributed learning over \emph{local learning}, in which nodes individually solve the empirical risk minimization problem using only their local data samples. Prior work on stochastic convex optimization (see, e.g.,~\cite[Theorem~5~and~(11)]{shalev2009stochastic}) tells us that, with high probability and ignoring the $\log$ terms, the gap between the statistical risk achieved by the empirical risk minimizer and the statistical risk minimizer scales as $\mathcal{O}\left(1/\sqrt{\text{\# of samples}}\right)$ in the centralized setting. This \emph{learning rate} translates into $\mathcal{O}(1/\sqrt{N})$ for local learning and $\mathcal{O}(1/\sqrt{MN})$ for the idealized centralized learning. In contrast, Theorem~\ref{coordinate descent theorem} can be interpreted as resulting in the following \emph{distributed} learning rate (with high probability):\footnote{We are once again ignoring the $\log$ terms in our discussion; it can be checked, however, that the $\log$ terms resulting from Theorem~\ref{coordinate descent theorem} match the ones in prior works such as \cite{shalev2009stochastic} on centralized learning.}
\begin{align}\label{eqn:effective.learning.rate}
  \mathbb{E}[f(w_j^{\bar{r},T},(x,y))] - \mathbb{E}[f(w^\ast,(x,y)) = \mathcal{O}\left(1/\sqrt{N_\text{eff}}\right),
\end{align}
where $N_\text{eff} := N/\bar{a}^2$ denotes the \emph{effective} number of training samples available during distributed learning.

In order to understand the significance of \eqref{eqn:effective.learning.rate}, notice that
\begin{align}\label{eqn:effective.number.samples}
  \frac{1}{M} \leq \frac{1}{\vert J'\vert} \leq \bar{a}^2 \leq 1 \ \Rightarrow \ N \leq N_\text{eff} \leq |J'| N \leq NM.
\end{align}
In particular, $\bar{a}^2 = 1/M$ if and only if there are no Byzantine failures in the network, resulting in the coordinate descent-based distributed learning of $\mathcal{O}\left(1/\sqrt{N M}\right)$, which matches the centralized learning rate. (This, to the best of our knowledge, is the first result on the explicit learning rate of coordinate descent-based distributed learning.) In the presence of Byzantine nodes, however, the \emph{maximum} number of \emph{trustworthy} samples in the network is $|J'| N$, and \eqref{eqn:effective.learning.rate} and \eqref{eqn:effective.number.samples} tell us that the learning rate of ByRDiE in this scenario will be somewhere between the idealized learning rate of $\mathcal{O}\left(1/\sqrt{|J'|N}\right)$ and the local learning rate of $\mathcal{O}\left(1/\sqrt{N}\right)$.

Our discussion so far has focused on the rate of statistical convergence (i.e., learning rate) of ByRDiE. The proof of Theorem~\ref{coordinate descent theorem}, however, also contains within itself the algorithmic rate of convergence for ByRDiE. We state this convergence rate in terms of the following theorem, which uses the notation $\bar{f}^0$ to denote the starting statistical risk $\mathbb{E}[f(0,(x,y))]$, $\bar{f}^*$ to denote the minimum statistical risk $\mathbb{E}[f(w^\ast,(x,y))]$ and $1 - \delta(\epsilon, N, \bar{a})$ to express the probability expression in \eqref{thm:byrdie.convergence.caseI.prob}.
\begin{theorem}[Algorithmic Convergence Rate for ByRDiE]\label{thm:algorithmic.convergence.rate}
Let Assumptions~\ref{assumption lipschitz}--\ref{assumption reduced graph} hold. Then, $\forall j \in J', \forall r \in \mathbb{N}, \forall \epsilon > 0,$ and $T \to \infty,$ we get with probability exceeding $1 - \delta(\epsilon, N, \bar{a})$ that
\begin{align}\label{thm:byrdie.algo.convergence.caseI}
  &\mathbb{E}\left[f(w_j^{r,T},(x,y))\right] - \bar{f}^* <\nonumber\\
    &\qquad\qquad\qquad\qquad\max\left\{\left(\bar{f}^0 - \bar{f}^*\right)\left(1 - \frac{r \epsilon}{c_5}\right), \epsilon\right\},
\end{align}
where the parameter $\delta(\epsilon, N, \bar{a})$ is given by
\begin{align}
2\exp\left(-\frac{4 |J'|N \epsilon^2}{{c_1'}^2 |J'| \bar{a}^2 + \epsilon^2} + |J'|\log\left(\frac{c_2'}{\epsilon}\right) + P\log\left(\frac{c_3'}{\epsilon}\right)\right),
\end{align}
while $c_5 := c_4 L' \sqrt{P} \gamma^*$, and the constants $c_1', c_2', c_3'$, and $c_4$ are as defined in Theorem~\ref{coordinate descent theorem}.
\end{theorem}
The proof of this theorem is given in Appendix~\ref{proof.thm:algorithmic.convergence.rate}. It can be seen from Theorem~\ref{thm:algorithmic.convergence.rate} that, with high probability, ByRDiE requires $r = \mathcal{O}\left(1/\epsilon\right)$ iterations to bring the \emph{excess risk} $\mathbb{E}[f(w_j^{r,T},(x,y))] - \bar{f}^*$ down to $\epsilon$. In terms of minimization of the statistical risk, therefore, ByRDiE achieves a sublinear rate of algorithmic convergence with high probability, even in the presence of Byzantine failures in the network.

\subsection{Theoretical Guarantees: Convergence for $T=1$}
Case~I for ByRDiE, in which $T \to \infty$, is akin to doing exact line search during minimization of each coordinate, which is one of the classic ways of implementing coordinate descent. Another well-adopted way of performing coordinate descent is to take only one step in the direction of descent in a dimension and then switch to another dimension~\cite{wright2015coordinate}. Within the context of ByRDiE, this is equivalent to setting $T=1$ (Case~II); our goal here is to provide convergence guarantees for ByRDiE in this case. Our analysis in this section uses the compact notation $w_j^r := w_j^{r,1} \equiv w_j^{r+1}(1)$. In the interest of space, and since the probabilistic analysis of this section is similar to that of the previous section, our probability results are stated asymptotically here, rather than in terms of precise bounds.

The starting point of our discussion is Theorem~\ref{lemma consensus}. Recall from Sec.~\ref{ssec:consensus} the definition of the index $m := (r-1)T + t$. When $T=1$, we have $r=m$ and therefore $[w^r_j]_k\to v^k(r)$ as $r\to\infty$ according to Theorem~\ref{lemma consensus}. In order to provide convergence guarantee, we only need to show that $v^k(r) \xrightarrow{r,N} [w^\ast]_k$ in probability. To this end, we define a sequence $Q(q) \in \mathbb{R}^P$ as follows: $\forall k, [Q(1)]_k:=v^k(1)$, while for the parameterized index $q=(r-1)P+k \neq 1$, $Q(q)$ is obtained by replacing $[Q(q-1)]_k$ with $v^k(r)$ and keeping the other dimensions fixed. Similarly, we define a sequence $\eta(q)$ satisfying $\eta(q)=\rho(r)$ for $Pr\leq q<(P+1)r$. Notice that $0<\eta(q+1)\leq\eta(q)$, $\sum_q \eta(q)=\infty$ and $\sum_q \eta^2(q)<\infty$. Since we have from \eqref{define vt} that $v^k(r+1)=v^k(r)-\rho(r)\sum_{i=1}^{\vert J'\vert}[\pi(r+1)]_i[\nabla \widehat{f}(w_i^r,S_i)]_k$, we can write the following iteration:
\begin{align}\label{update Q}
  Q(q+1)=Q(q)-\eta(q)\sum\limits_{i=1}^{\vert J'\vert}[\pi(r+1)]_i[\nabla \widehat{f}(w_i^r,S_i)]_ke_k,
\end{align}
where $e_k$ denotes the standard basis vector (i.e., it is zero in every dimension except $k$ and $[e_k]_k=1$) and the relationship between $r,k$, and $q$ is as defined earlier. The sequence $Q(q)$ effectively helps capture update of the optimization variable after each coordinate-wise update of ByRDiE. In particular, we have the following result concerning the sequence $Q(q)$.

\begin{lemma}\label{lemma Q converge}
Let Assumptions \ref{assumption lipschitz}, \ref{assumption finite value}, and \ref{assumption reduced graph} hold for ByRDiE and choose $T=1$. We then have that $Q(q) \xrightarrow{q,N} w^\ast$ in probability.
\end{lemma}
The proof of this lemma is provided in Appendix~\ref{proof of lemma Q convergence}. We are now ready to state the convergence result for ByRDiE under Case~II (i.e, $T=1$).
\begin{theorem}[Asymptotic Convergence of ByRDiE]\label{theorem vector case 2}
Let Assumptions \ref{assumption lipschitz}, \ref{assumption finite value}, and \ref{assumption reduced graph} hold for ByRDiE and choose $T=1$. Then, $\forall j \in J'$, $w_j^{\bar{r}} \xrightarrow{\bar{r},N} w^\ast$ in probability.
\end{theorem}
\begin{proof}
We have from Theorem~\ref{lemma consensus} that $\forall j \in J', [w_j^{\bar{r}}]_k \xrightarrow{\bar{r}} v^k(\bar{r})$ for all $k \in \{1,2,\dots,P\}$. The definition of $Q(q)$ along with Lemma~\ref{lemma Q converge} also implies that $v^k(\bar{r}) \xrightarrow{\bar{r},N} [w^\ast]_k$ in probability. This completes the proof of the theorem.
\end{proof}

\subsection{How to Choose the Parameter $T$ in ByRDiE?}\label{ssec:choice.T}
The parameter $T$ in ByRDiE trades off consensus among the nonfaulty nodes and the convergence rate as a function of the number of \emph{communication iterations} $t_c(r,k,t)$, defined as
\begin{align}
  t_c := (r-1)TP + [(k-1)T + t].
\end{align}
In particular, given a fixed $T$, each iteration $r$ of ByRDiE involves $TP$ (scalar-valued) communication exchanges among the neighboring nodes. In the previous section, we have provided theoretical guarantees for the two extreme cases of $T\to\infty$ and $T=1$. In the limit of large $\bar{r}$, our results establish that both extremes result in consensus and convergence to the statistical risk minimizer. In practice, however, different choices of $T$ result in different behaviors as a function of $t_c (\equiv t_c(r,k,t))$, as discussed in the following and as illustrated in our numerical experiments in the next section.

When $T$ is large, the two time-scale nature of ByRDiE ensures the disagreement between nonfaulty nodes does not become too large in the initial stages of the algorithm; in particular, the larger the number of iterations $T$ in the inner loop, the smaller the disagreement among the nonfaulty nodes at the beginning of the algorithm. Nonetheless, this comes at the expense of slower convergence to the desired minimizer as a function of the number of communication iterations $t_c$.

On the other hand, while choosing $T=1$ also guarantees consensus among nonfaulty nodes, it only does so asymptotically (cf.~Theorem~\ref{lemma consensus}). Stated differently, ByRDiE cannot guarantee in this case that the disagreement between nonfaulty nodes will be small in the initial stages of the algorithm (cf.~Theorem~\ref{lemma:consensus.rate}). This tradeoff between small consensus error and slower convergence (as a function of communication iterations $t_c$) should be considered by a practitioner when deciding the value of $T$. We conclude by noting that the different nature of the two extreme cases requires markedly different proof techniques, which should be of independent interest to researchers.

\begin{remark}
Our discussion so far has focused on the use of a \emph{static} parameter $T$ within ByRDiE. Nonetheless, it is plausible that one could achieve somewhat better tradeoffs between consensus and convergence through the use of an adaptive parameter $T_r$ in lieu of $T$ that starts with a large value and gradually decreases as $r$ increases. Careful analysis and investigation of such an adaptive two-time scale variant of ByRDiE, however, is beyond the scope of this paper.
\end{remark}

\section{Numerical Results}\label{section numerical}
In this section, we validate our theoretical results and make various observations about the performance of ByRDiE using two sets of numerical experiments. The first set of experiments involves learning of a binary classification model from the \emph{infinite MNIST} dataset\footnote{\tt https://leon.bottou.org/projects/infimnist} that is distributed across a network of nodes. This set of experiments fully satisfies all the assumptions in the theoretical analysis of ByRDiE. The second set of experiments involves training of a small-scale neural network for classification of the \emph{Iris} dataset~\cite{dua2017} distributed across a network. The learning problem in this case corresponds to a nonconvex one, which means this set of experiments does not satisfy the main assumptions of our theorems. Nonetheless, we show in the following that ByRDiE continues to perform well in such distributed nonconvex learning problems.

\subsection{Distributed SVM Using Infinite MNIST Dataset}
We consider a distributed linear binary classification problem involving MNIST handwritten digits dataset. The (infinite) MNIST dataset comprises images of handwritten digits from `0' to `9'. Since our goal is demonstration of the usefulness of ByRDiE in the presence of Byzantine failures, we focus only on distributed training of a linear support vector machine (SVM) for classification between digits `5' and `8', which tend to be the two most inseparable digits. In addition to highlighting the robustness of ByRDiE against Byzantine failures in this problem setting, we evaluate its performance for different choices of the parameters $T$, $N$, and $b$.

\begin{figure}[t]
\centering
\includegraphics[width=.95\columnwidth]{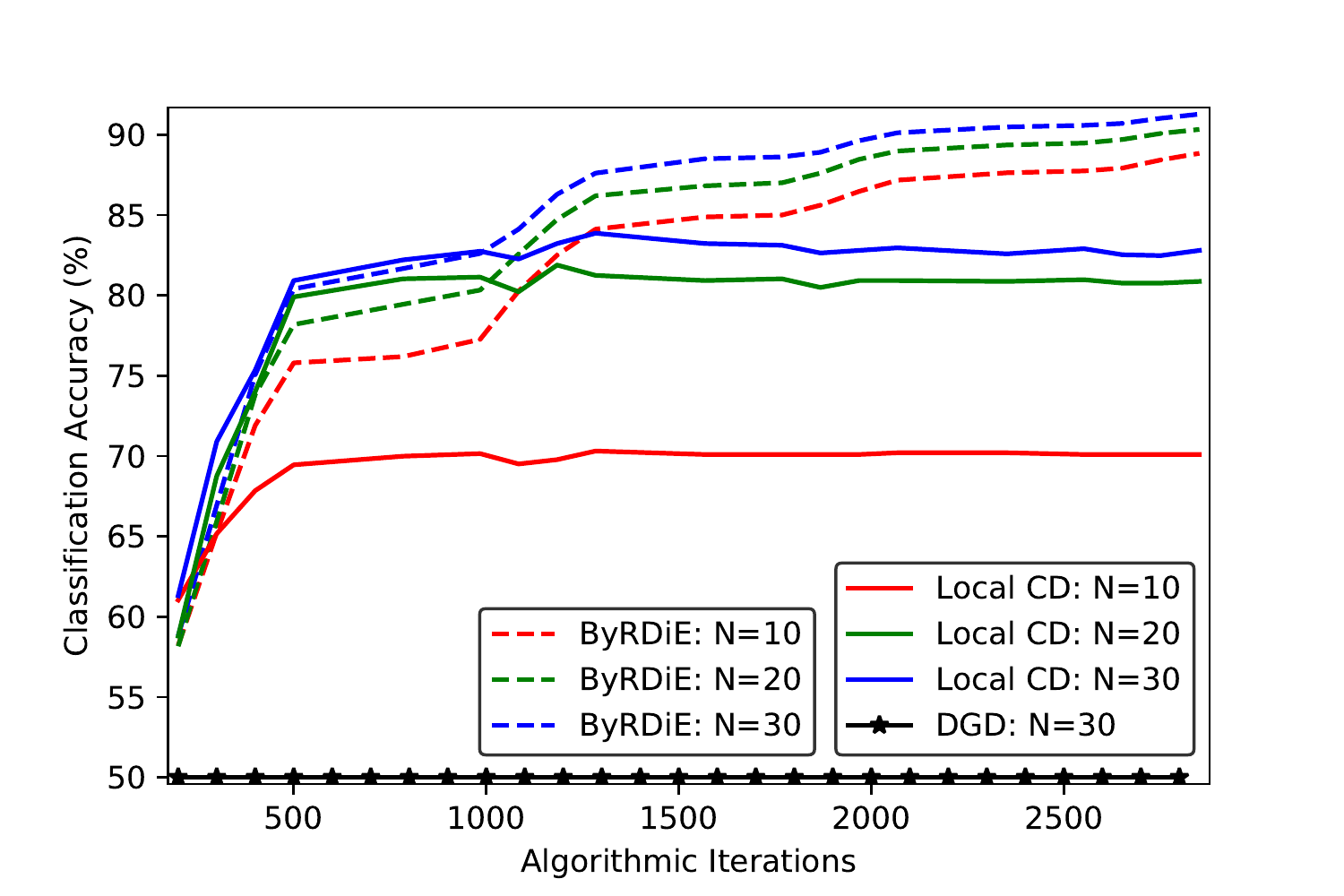}
\caption{Average classification accuracy of ByRDiE, local CD, and DGD (on test data) for different values of $N$ as a function of the number of algorithmic iterations for distributed training of a linear SVM on (infinite) MNIST dataset.}\label{fig: varying_N}
\end{figure}

In terms of the experimental setup, we generate Erd\H{o}s--R\'{e}nyi networks ($p = 0.5$) of $M$ nodes, $b$ of which are randomly chosen to be Byzantine nodes. All nonfaulty nodes are allocated $N$ samples---equally divided between the two classes---from the dataset, while a Byzantine node broadcasts random data uniformly distributed between $0$ and $1$ to its neighbors in each iteration. When running ByRDiE algorithm, each node updates one dimension $T$ times before proceeding to the next dimension. All tests are performed on the same test set with 1000 samples of digits `5' and `8' each.

\begin{figure}[t]
\centering
\begin{subfigure}{.95\columnwidth}
  \centering
  \includegraphics[width=\linewidth]{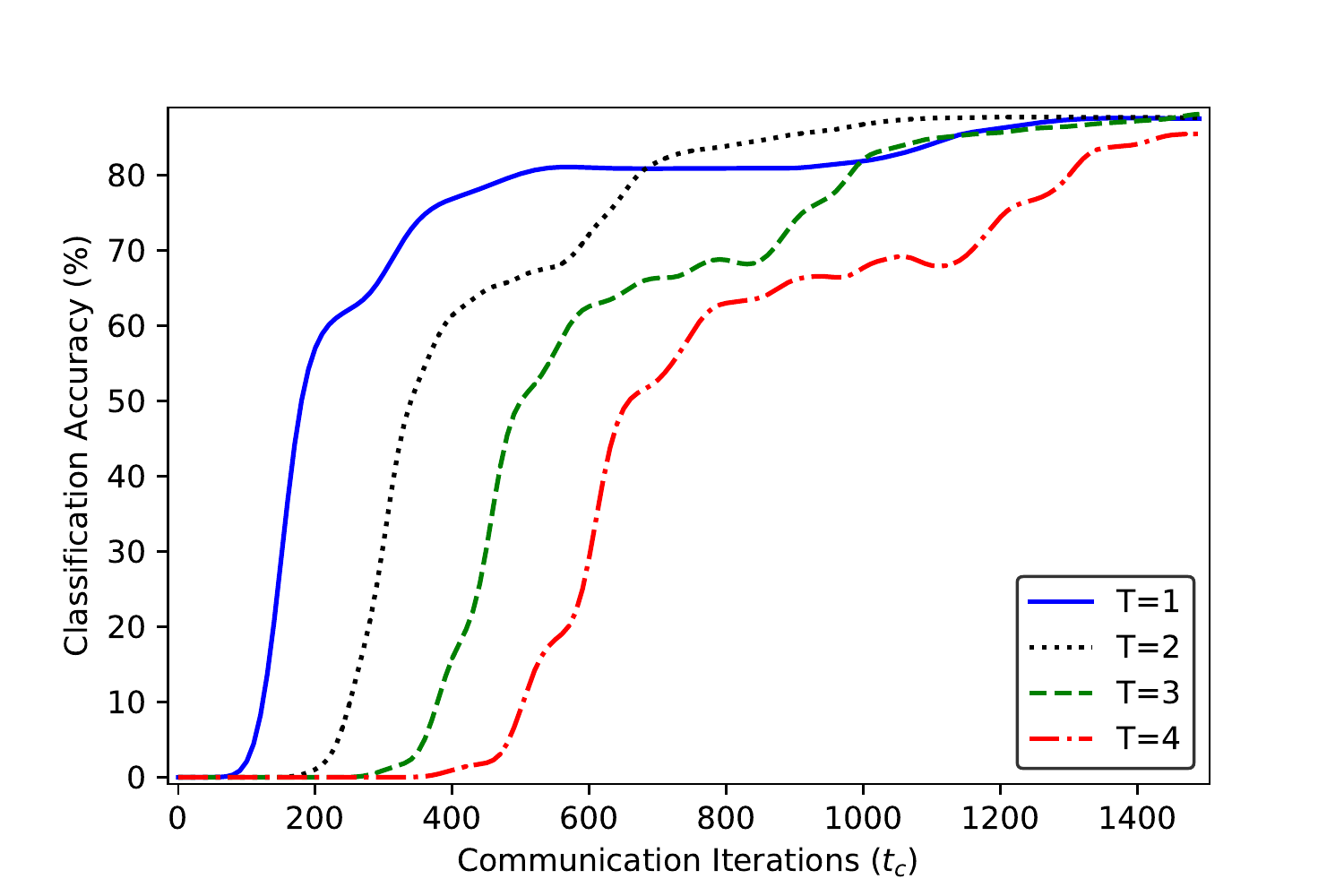}
  \caption{Classification accuracy for different values of $T$}
  \label{fig: varying_T_accuracy}
\end{subfigure}
\\
\vspace{\baselineskip}
\begin{subfigure}{.95\columnwidth}
  \centering
  \includegraphics[width=\linewidth]{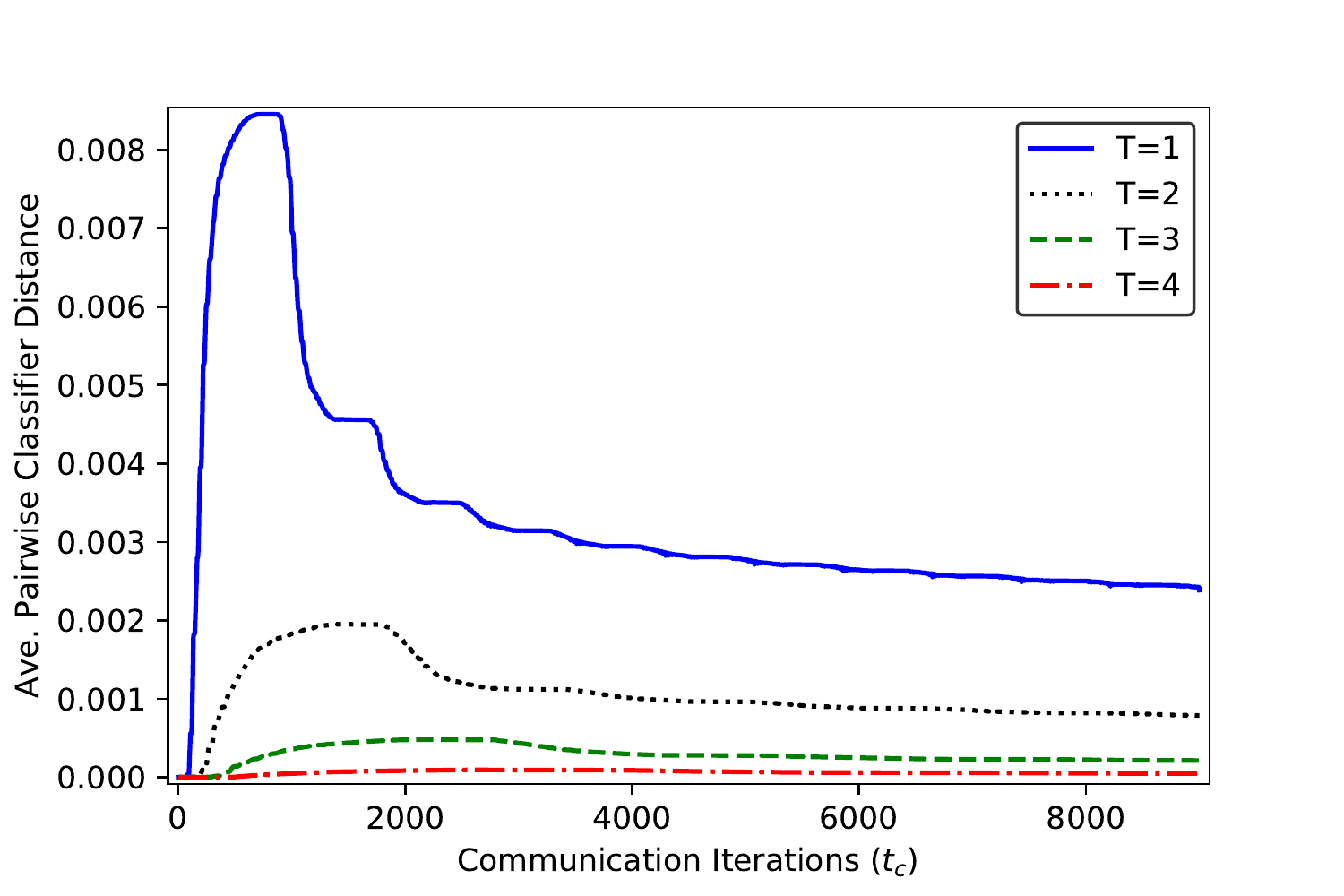}
  \caption{Consensus behavior for different values of $T$}
  \label{fig: varying_T_consensus}
\end{subfigure}
\caption{Convergence and consensus behavior of ByRDiE as a function of the number of communication iterations $t_c$ for different values of the parameter $T$. All plots correspond to distributed training of a linear SVM using the MNIST dataset.}
\label{fig: varying_T}
\end{figure}

We first report results that confirm the idea that ByRDiE can take advantage of cooperation among different nodes to achieve better performance even when there are Byzantine failures in the network. This involves varying the local sample size $N$ and comparing the classification accuracy on the test data. We generate a network of $M = 50$ nodes, randomly pick $b=10$ nodes within the network to be Byzantine nodes ($20\%$ failures), vary $N$ from $10$ to $30$, and average the final set of results over 10 independent (over network, Byzantine nodes, and data allocation) Monte Carlo trials of ByRDiE. The performance of ByRDiE is compared with two approaches: ($i$) coordinate descent-based training using only local data (local CD); and ($ii$) distributed gradient descent-based~\cite{Nedic2009Distributed} training involving network data (DGD). To achieve the best convergence rate for ByRDiE, $T$ is chosen to be 1 in these experiments. The final set of results are shown in Fig~\ref{fig: varying_N}, in which the average classification accuracy is plotted against the number of algorithmic iterations, corresponding to the number of (scalar-valued) communication iterations for ByRDiE, the number of per-dimension updates for local CD, and the number of (vector-valued) communication iterations for DGD. It can be seen that the performance of local CD is not good enough due to the small local sample size. On the other hand, when trying to improve performance by cooperating among different nodes, DGD fails for lack of robustness against Byzantine failures. In contrast, the higher accuracy of ByRDiE shows that ByRDiE can take advantage of the larger distributed dataset while being Byzantine resilient.

Next, we investigate the impact of different values of $T$ in ByRDiE on the tradeoff between consensus and convergence rate. This involves generating a network of $M=50$ nodes that includes randomly placed $b=5$ Byzantine nodes within the network ($10\%$ failures), randomly allocating $N = 60$ training samples to each nonfaulty node, and averaging the final set of results over $10$ independent trials. The corresponding results are reported in Fig.~\ref{fig: varying_T} for four different values of $T$ as a function of the number of communication iterations $t_c$ in terms of ($i$) average classification accuracy (Fig.~\ref{fig: varying_T_accuracy}) and ($ii$) average pairwise distances between local classifiers (Fig.~\ref{fig: varying_T_consensus}). It can be seen from these figures that $T=1$ leads to the fastest convergence in the initial stages of the algorithm, but this fast convergence comes at the expense of the largest differences among local classifiers, especially in the beginning of the algorithm. In contrast, while $T = 4$ results in the slowest convergence, it ensures closeness of the local classifiers at all stages of the algorithm.

\begin{figure}[t]
\centering
\begin{subfigure}{.95\columnwidth}
  \centering
  \includegraphics[width=\linewidth]{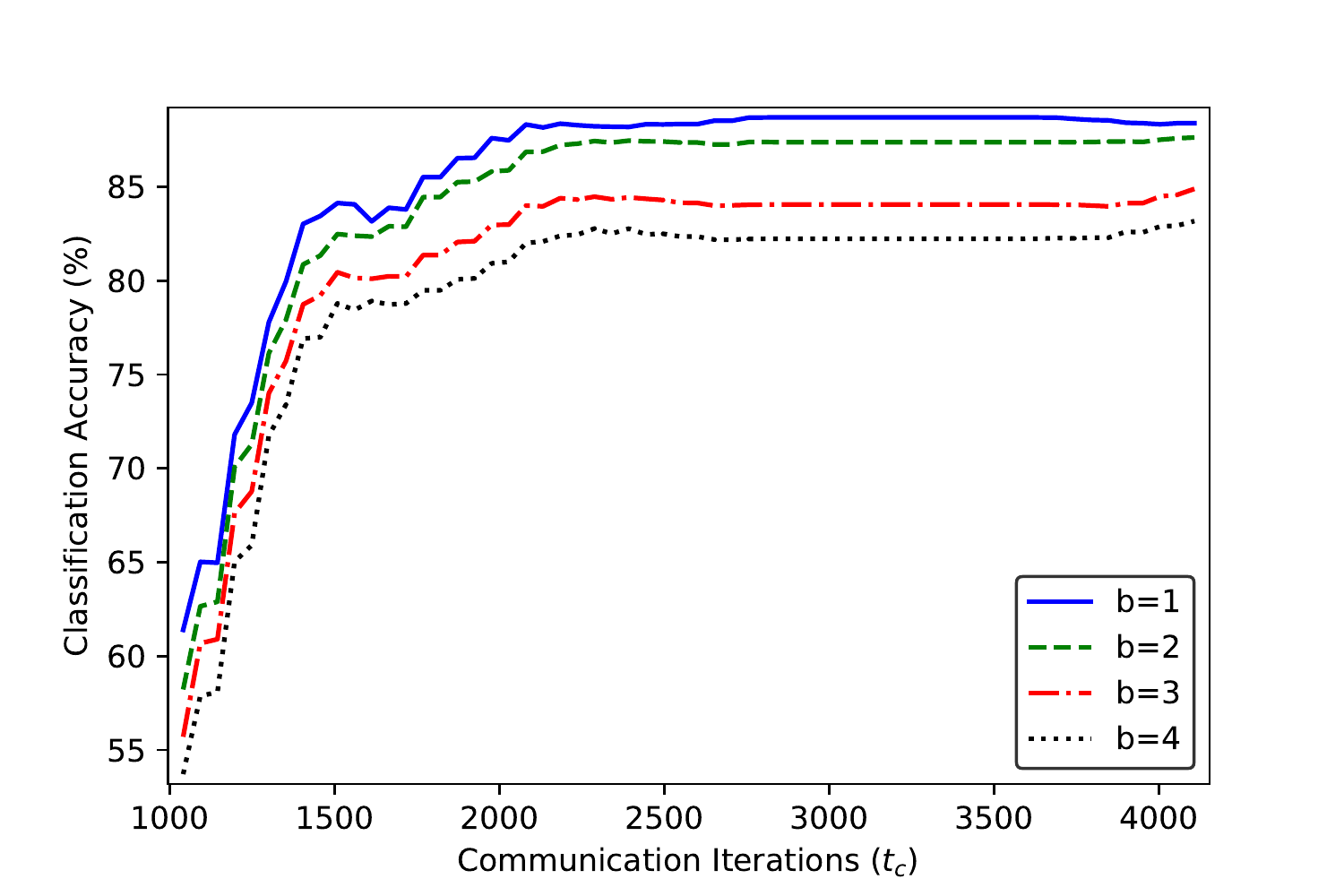}
  \caption{Classification accuracy for different values of $b$}
  \label{fig: varying_b_accuracy_big}
\end{subfigure}
\\
\vspace{\baselineskip}
\begin{subfigure}{.95\columnwidth}
  \centering
  \includegraphics[width=\linewidth]{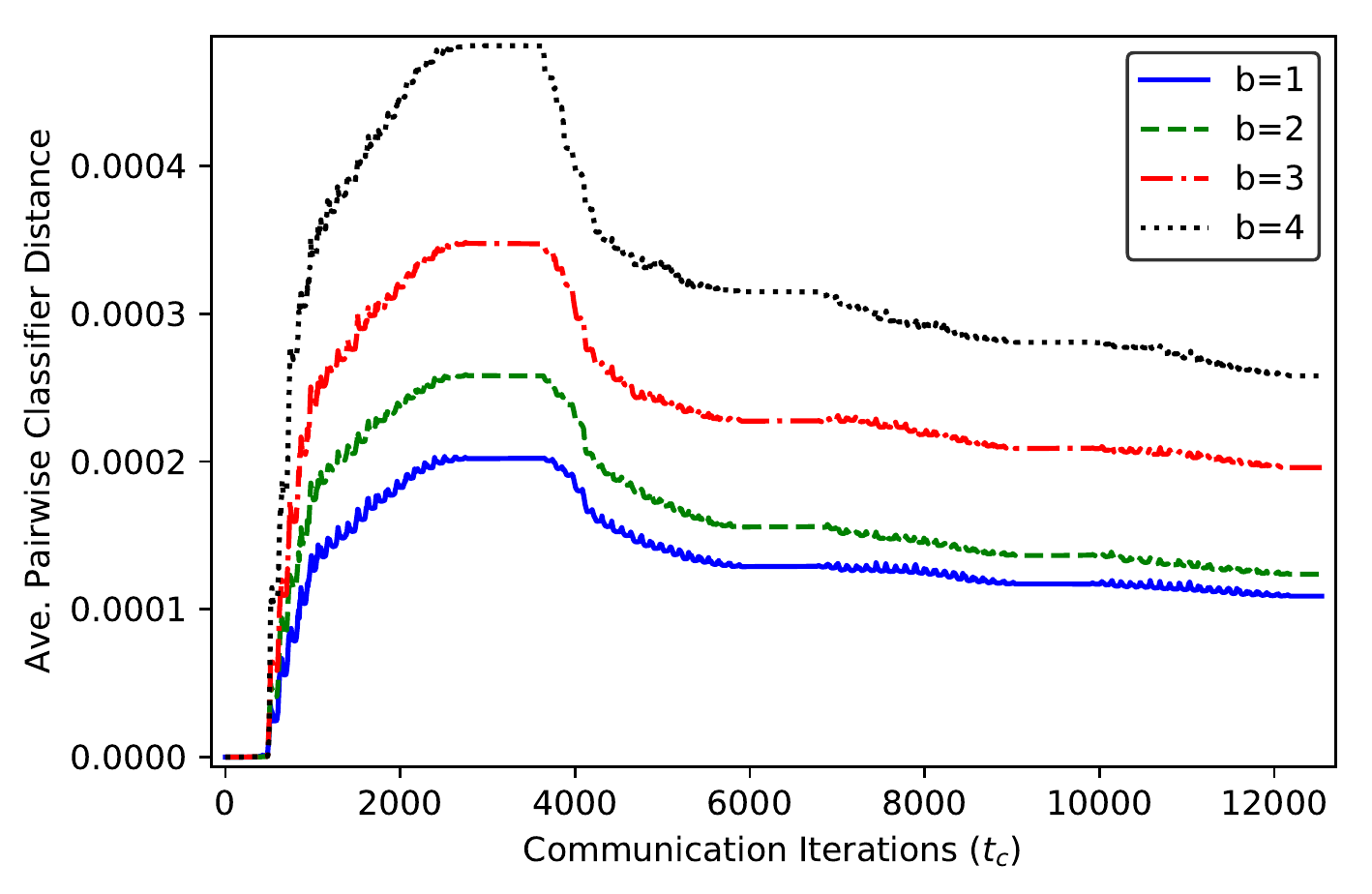}
  \caption{Consensus behavior for different values of $b$}
  \label{fig: varying_b_consensus}
\end{subfigure}
\caption{Convergence and consensus behavior of ByRDiE within a small network ($M=20$) as a function of the number of communication iterations $t_c$ for different number $b$ of Byzantine nodes in the network. All plots correspond to
distributed training of a linear SVM using the MNIST dataset.}
\label{fig: varying_b_big}
\end{figure}

Finally, we investigate the impact of different values of $b$ (the actual number of Byzantine nodes) on the performance of ByRDiE. In order to amplify this impact, we focus on a smaller network of $M = 20$ nodes with a small number of $N = 10$ samples per node and $T = 3$. Assuming resilience against the worst-case Byzantine scenario (cf.~Remark~\ref{rem:choice.b}), ByRDiE requires the neighborhood of each node to have at least $(2b+1)$ nodes. Under this constraint, we find that $b \geq 5$ in this setting, which translates into $25\%$ or more of the nodes as being Byzantine, often renders ByRDiE unusable.\footnote{As noted in Remark~\ref{rem:choice.b}, one could overcome this limit by opting instead for resilience against the average-case scenario and replacing $b$ with $b_{\text{ave}}$.} We therefore report our results in terms of both consensus and convergence behavior by varying $b$ from $1$ to $4$. The final results, averaged over $10$ independent trials, are shown in Fig.~\ref{fig: varying_b_big}. It can be seen from these figures that both the classification accuracy (Fig.~\ref{fig: varying_b_accuracy_big}) and the consensus performance (Fig.~\ref{fig: varying_b_consensus}) of ByRDiE suffer as $b$ increases from $1$ to $4$. Such behavior, however, is in line with our theoretical developments. First, as $b$ increases, the \emph{post-screening} graph becomes sparser, which slows down information diffusion and consensus. Second, as $b$ increases, fewer samples are incorporated into distributed learning, which limits the final classification accuracy of distributed SVM.

\subsection{Distributed Neural Network Using Iris Dataset}
While the theoretical guarantees for ByRDiE have been developed for convex learning problems, we now demonstrate the usefulness of ByRDiE in distributed nonconvex learning problems. Our setup in this regard involves distributed training of a small-scale neural network using the three-class Iris dataset. This dataset consists of 50 samples each from three species of irises, with each sample being a four-dimensional feature vector. While one of the species in this data is known to be linearly separable from the other two, the remaining two species cannot be linearly separated. We employ a fully connected three-layer neural network for classification of this dataset, with a four-neuron input layer, a three-neuron hidden layer utilizing the \emph{rectified linear unit} (ReLU) activation function and a three-neuron output layer using the softmax function for classification. The distributed setup corresponds to a random network of $M = 10$ nodes with one Byzantine node ($b=1$), in which each nonfaulty node is allocated $N = 15$ samples that are equally divided between the three classes.

\begin{table}[h]
\renewcommand{\arraystretch}{1.2}
\centering
\begin{tabular}{|c|c|c|}
\hline
{\bf Algorithm}  & {\bf Iterations to achieve $95\%$ accuracy} & {\bf Consensus}  \\
\hline
DGD  & $\infty$ & No  \\
\hline
ByRDiE  & 19 & Yes  \\
\hline\hline
Centralized CD& 17 & N/A  \\
\hline
\end{tabular}
\caption{Distributed training of a neural network in the presence of Byzantine failures using ByRDiE and DGD.}
\label{table: Iris}
\end{table}

Since distributed training of this neural network requires solving a distributed nonconvex problem, our theoretical guarantees for ByRDiE do not hold in this setting. Still, we use ByRDiE with $T=1$ to train the neural network for a total of 200 independent trials with independent random initializations. Simultaneously, we use DGD for distributed training and also train the neural network in a centralized setting using CD for comparison purposes. The final results are reported in Table~\ref{table: Iris}, which lists the average number of \emph{outer} iterations needed by each training algorithm to achieve $95\%$ classification accuracy. It can be seen from these results that while DGD completely breaks down in the presence of a \emph{single} Byzantine node, ByRDiE continues to be resilient to Byzantine failures in distributed nonconvex learning problems and comes close to matching the performance of centralized CD in this case.

\section{Conclusion}\label{section conclusion}
In this paper, we have introduced a coordinate descent-based distributed algorithm, termed ByRDiE, that can carry out distributed learning tasks in the presence of Byzantine failures in the network. The proposed algorithm pursues the minimizer of the statistical risk in a fully distributed environment and theoretical results guarantee that ByRDiE achieves this objective under mild assumptions on the risk function and the network topology. In addition, numerical results presented in the paper validate resilience of ByRDiE against Byzantine failures in the network for distributed convex and nonconvex learning tasks. Future works in this direction include strengthening the results in terms of almost sure convergence, deriving explicit rates of convergence, and extensions of theoretical analysis to constant step size and nonconvex risk functions.

\appendices
\section{Proof of Theorem~\ref{lemma consensus}}\label{proof of lemma consensus}
\noindent Fix any dimension $k \in \{1,\dots,P\}$ and recall from \eqref{define vt} that
\begin{align}\label{define vt+1}
  &V(m+1)=\mathbf{1}\pi^T(1)\Omega(1)\nonumber\\ &\quad-\sum\limits_{\tau=1}^{m-1}\mathbf{1}\pi^T(\tau+1)\bar{\rho}(\tau)G(\tau)-\mathbf{1}\pi^T(m+1)\bar{\rho}(m)G(m).
\end{align}
Since $m := (r-1)T+t$ and $[\Omega(m)]_j :=[w_j^r(t)]_k$, we get
\begin{align}
 &v^k(m+1)\nonumber\\
    &\quad=\sum\limits_{i=1}^{\vert J'\vert}[\pi(1)]_i[w_i^1(1)]_k-\sum\limits_{\tau=1}^{m-1}\bar{\rho}(\tau)\sum\limits_{i=1}^{\vert J'\vert}[\pi(\tau+1)]_i[G(\tau)]_i\nonumber\\
    &\qquad\qquad-\bar{\rho}(m)\sum\limits_{i=1}^{\vert J'\vert}[\pi(r+1)]_i[G(m)]_i.
\end{align}
We also get from \eqref{expression of wt+1 from w0} that
\begin{align}
 &[w_j^r(t+1)]_k=\sum\limits_{i=1}^{\vert J'\vert}[\Phi(m,1)]_{ji}[w^1_i(1)]_k \nonumber\\ &\quad-\sum\limits_{\tau=1}^{m-1}\bar{\rho}(\tau)\sum\limits_{i=1}^{\vert J'\vert}[\Phi(m,\tau+1)]_{ji}[G(\tau)]_i-\bar{\rho}(m)[G(m)]_j,
\end{align}
where we have used the fact that $\Phi(m,m+1)=I$. Thus,
\begin{align}
&\Big|[w_j^r(t+1)]_k-v^k(m+1)\Big|=\nonumber\\
&\quad\Big|\sum\limits_{i=1}^{\vert J'\vert}([\Phi(m,1)]_{ji}-[\pi(1)]_i)[w_i^1(1)]_k\nonumber\\
&\qquad-\sum\limits_{\tau=1}^{m-1}\bar{\rho}(\tau)\sum\limits_{i=1}^{\vert J'\vert}([\Phi(m,\tau+1)]_{ji}-[\pi(\tau+1)]_i)[G(\tau)]_i\nonumber\\
&\qquad\quad-\bar{\rho}(m)\sum\limits_{i=1}^{\vert J'\vert}[\pi(m+1)]_i([G(m)]_j-[G(m)]_i)\Big|,
\end{align}
where the last summation follows from the observation that $\sum_{i=1}^{\vert J'\vert}[\pi(m+1)]_i=1$. Further, Assumption~\ref{assumption lipschitz} implies there exists a coordinate-wise bound $L_\nabla \leq L'$ such that $\forall i,m, |[G(m)]_i| \leq L_\nabla$. Using this and Lemma~\ref{lemma property 2 of phi}, we obtain
\begin{align}\label{consensus rate}
&\Big|[w_j^r(t+1)]_k-v^k(m+1)\Big|\nonumber\\
&\leq\sum\limits_{i=1}^{\vert J'\vert}\Big|[\Phi(m,1)]_{ji}-[\pi(1)]_i\Big|\,\big| [w_i^1(1)]_k\big|\nonumber\\
&\quad +\sum\limits_{\tau=1}^{m-1}\bar{\rho}(\tau)\sum\limits_{i=1}^{\vert J'\vert}\Big|[\Phi(m,\tau+1)]_{ji} - [\pi(\tau+1)]_i\Big|\big|[G(\tau)]_i\big|\nonumber\\
&\quad+\bar{\rho}(m)\sum\limits_{i=1}^{\vert J'\vert}[\pi(m+1)]_i\Big|[G(m)]_j-[G(m)]_i\Big|\nonumber\\
&\leq \vert J'\vert L_\nabla\sum\limits_{\tau=1}^{m-1}\bar{\rho}(\tau)\mu^{\frac{m-\tau}{\nu}}+2\bar{\rho}(m)L_\nabla.
\end{align}
Note that the last inequality in the above expression exploits the fact that $[w_i^1(1)]_k \equiv 0$. In the case of an arbitrary initialization of ByRDiE, however, we could have bounded the first term in the second inequality in \eqref{consensus rate} as $\vert J'\vert \Gamma \mu^{\frac{m+1}{\nu}}$, which still goes to zero as $m \rightarrow \infty$.

We now expand the term $\sum_{\tau=1}^{m-1}\bar{\rho}(\tau)\mu^{\frac{m-\tau}{\nu}}$ in \eqref{consensus rate} as\footnote{We focus here only on the case of an even $m$ for the sake of brevity; the expansion for an odd $m$ follows in a similar fashion.}
\begin{align*}
\sum_{\tau=1}^{m-1}\bar{\rho}(\tau)\mu^{\frac{m-\tau}{\nu}} &= \left(\sum\limits_{\tau=1}^{\frac{m}{2}-1}\bar{\rho}(\tau)\mu^{\frac{m-\tau}{\nu}}+\sum\limits_{\tau=\frac{m}{2}}^{m-1}\bar{\rho}(\tau)\mu^{\frac{m-\tau}{\nu}}\right)\nonumber\\
&\leq\bar{\rho}(1)\sum\limits_{\tau=1}^{\frac{m}{2}-1}\mu^{\frac{m-\tau}{\nu}}+\bar{\rho}(\tfrac{r}{2})\sum\limits_{\tau=\frac{m}{2}}^{m-1}\mu^{\frac{m-\tau}{\nu}}\nonumber\\
&=\frac{\bar{\rho}(1)\mu^{\frac{\frac{m}{2}+1}{\nu}}(1-\mu^{\frac{m}{2\nu}})}{\mu^{1-\frac{1}{\nu}}}+\frac{\bar{\rho}(\frac{m}{2})\mu^{\frac{1}{\nu}}(1-\mu^{\frac{m}{2\nu}})}{\mu^{1-\frac{1}{\nu}}}.
\end{align*}
It can be seen from the above expression that $\lim_{m \rightarrow \infty}\sum_{\tau=1}^{m-1}\bar{\rho}(\tau)\mu^{\frac{m-\tau}{\nu}} \leq 0$. Since $\sum_{\tau=1}^{m-1}\bar{\rho}(\tau)\mu^{\frac{m-\tau}{\nu}}\geq 0$, it follows that $\lim_{m \rightarrow \infty}\sum_{\tau=1}^{m-1}\bar{\rho}(\tau)\mu^{\frac{m-\tau}{\nu}} = 0$. This fact in concert with \eqref{consensus rate} and the diminishing nature of $\bar{\rho}(m)$ give
\begin{align}\label{eqn:proof.consensus rate.2}
  \lim_{m\rightarrow \infty}\left\vert [w_j^r(t+1)]_k-v^k(m+1)\right\vert = 0.
\end{align}
Next, take $r = \bar{r}$, $t = T$, and note that $w_j^r(t+1) = w_j^{\bar{r},T}$ and $\bar{m} = \bar{r}T +1 = [(r-1)T + T]+1 = m + 1$ in this case. Further, $\bar{m}\to\infty$ when $\bar{r}\to \infty$ and/or $T\to\infty$. It therefore follows from \eqref{eqn:proof.consensus rate.2} that $[w_j^{\bar{r},T}]_k \to v^k(\bar{m})$ as $\bar{r}\to\infty$ and/or $T\to\infty$. Finally, since $k$ in our analysis was arbitrary, the same holds for all $k=1,2,\dots,P$.\qed

\section{Proof of Theorem~\ref{lemma vector case high probability}}\label{proof of lemma vector case high probability}
\noindent We begin by fixing $w^\prime \in \mathbb{R}$ such that $|w^\prime| \leq \Gamma$ and defining $\alpha^r_k \in \mathbb{R}^{|J'|}$ as $\alpha^r_k := [\alpha_j(r,k): j \in J']$ and $\tilde{F}^r_k(w') \in \mathbb{R}^{|J'|}$ as $\tilde{F}^r_k(w') := [\widehat{f}(\tilde{w}_k^r\vert_{[\tilde{w}_k^r]_k=w'},S_j): j \in J']$. We can then write
\begin{align}
  H^r_k(w^\prime) = {\alpha^r_k}^T\tilde{F}^r_k(w').
\end{align}
Next, notice from~\eqref{eqn:empirical.risk}, \eqref{eqn:def.hkr}, and \eqref{eqn:def.Hkr} that
\begin{align}
  \forall (r,k), \ \mathbb{E}[H^r_k(w^\prime)] = h^r_k(w^\prime).
\end{align}
We now fix indices $(r,k)$ and note that the random variables $\alpha_j(r,k) \ell(\tilde{w}_k^r\vert_{[\tilde{w}_k^r]_k=w'},(x_{jn},y_{jn}))$ involved in the definition of the univariate function $H^r_k(w^\prime)$ are ($i$) independent due to the independence of the training samples, and ($ii$) bounded as $0\leq\alpha_j\ell(w,(x_{jn},y_{jn}))\leq \alpha_jC$ due to Assumption~\ref{assumption finite value}. Therefore, the following holds $\forall \epsilon' > 0$ due to Hoeffding's inequality~\cite{hoeffding1963probability}:
\begin{align}
&\mathbb{P}\left(\left\vert {\alpha^r_k}^T\tilde{F}^r_k(w')-h^r_k(w^\prime)\right\vert \geq \epsilon^\prime\right) \equiv\nonumber\\
&\ \ \! \mathbb{P}\left(\left\vert H^r_k(w^\prime)-h^r_k(w^\prime)\right\vert \geq \epsilon^\prime\right) \leq 2\exp\left(-\frac{2N{\epsilon^\prime}^2}{C^2 \|\alpha^r_k\|^2}\right)\label{eqn:lemma.proof.single.bound.1}.
\end{align}
Further, since the $|J'|$-dimensional vector $\alpha^r_k$ is an arbitrary element of the standard simplex, defined as
\begin{align}
  \Delta := \{v \in \mathbb{R}^{|J'|}: \sum_{j=1}^{|J'|} [v]_j = 1 \text{ and } \forall j, [v]_j \geq 0\},
\end{align}
the probability bound in \eqref{eqn:lemma.proof.single.bound.1} also holds for any $v \in \Delta$, i.e.,
\begin{align}\label{eqn:lemma.proof.single.bound.2}
\mathbb{P}\left(\left\vert {v}^T\tilde{F}^r_k(w')-h^r_k(w^\prime)\right\vert \geq \epsilon^\prime\right) \leq 2\exp\left(-\frac{2N{\epsilon^\prime}^2}{C^2 \|v\|^2}\right).
\end{align}

We now define the set $\mathcal{S}_\alpha := \{\alpha^r_k\}_{r,k=1}^{\bar{r},P}$. Our next goal is to leverage \eqref{eqn:lemma.proof.single.bound.2} and derive a probability bound similar to \eqref{eqn:lemma.proof.single.bound.1} that \emph{uniformly} holds for \emph{all} $v \in \mathcal{S}_\alpha$. To this end, let
\begin{align}
  \mathbb{C}_\xi := \{c_1,\dots,c_{d_\xi}\} \subset \Delta \quad \text{s.t.} \quad \mathcal{S}_\alpha \subseteq \bigcup_{q=1}^{d_\xi} B(c_q, \xi)
\end{align}
denote an $\xi$-covering of $\mathcal{S}_\alpha$ in terms of the $\ell_2$ norm and define $\bar{c} := \argmax_{c \in \mathbb{C}_\xi} \|c\|$. It then follows from~\eqref{eqn:lemma.proof.single.bound.2} and the union bound that
\begin{align}\label{eqn:lemma.proof.union.bound.1}
&\mathbb{P}\left(\sup_{c \in \mathbb{C}_\xi} \left\vert c^T \tilde{F}^r_k(w') -h^r_k(w^\prime)\right\vert \geq \epsilon^\prime\right)\nonumber\\
&\qquad\qquad\qquad\qquad\qquad\leq 2d_\xi \exp\left(-\frac{2N{\epsilon^\prime}^2}{C^2 \|\bar{c}\|^2}\right).
\end{align}
In addition, we have
\begin{align}\label{eqn:lemma.proof.sup.bounds.1}
  &\sup_{v \in \mathcal{S}_\alpha} \left|v^T \tilde{F}^r_k(w') - h^r_k(w^\prime)\right| \stackrel{(a)}{\leq} \sup_{c \in \mathbb{C}_\xi} \left|c^T \tilde{F}^r_k(w') - h^r_k(w^\prime)\right| +\nonumber\\
  &\qquad\qquad\qquad\qquad\quad\sup_{v \in \mathcal{S}_\alpha,c \in \mathbb{C}_\xi}\|v - c\| \|\tilde{F}^r_k(w')\|,
\end{align}
where ($a$) is due to triangle and Cauchy--Schwarz inequalities. Trivially, $\sup_{v \in \mathcal{S}_\alpha,c \in \mathbb{C}_\xi}\|v - c\| \leq \xi$ from the definition of $\mathbb{C}_\xi$, while $\|\tilde{F}^r_k(w')\| \leq \sqrt{|J'|}C$ from the definition of $\tilde{F}^r_k(w')$ and Assumption~\ref{assumption finite value}. Combining \eqref{eqn:lemma.proof.union.bound.1} and \eqref{eqn:lemma.proof.sup.bounds.1}, we get
\begin{align}\label{eqn:lemma.proof.union.bound.2}
&\mathbb{P}\left(\sup_{v \in \mathcal{S}_\alpha} \left\vert v^T \tilde{F}^r_k(w') -h^r_k(w^\prime)\right\vert \geq \epsilon^\prime + \sqrt{|J'|}\xi C\right)\nonumber\\
&\qquad\qquad\qquad\qquad\qquad\quad\leq 2d_\xi \exp\left(-\frac{2N{\epsilon^\prime}^2}{C^2 \|\bar{c}\|^2}\right).
\end{align}

We now define $\bar{\alpha} := \argmax_{v \in \mathcal{S}_\alpha} \|v\|$. It can then be shown from the definitions of $\mathbb{C}_\xi$ and $\bar{c}$ that
\begin{align}\label{eqn:max.alpha.norm}
  \|\bar{c}\|^2 \leq 2(\|\bar{\alpha}\|^2 + \xi^2).
\end{align}
Therefore, picking any $\epsilon^{\prime\prime} > 0$, and defining $\epsilon^\prime := \epsilon^{\prime\prime}/2$ and $\xi := \epsilon^{\prime\prime}/(2C\sqrt{|J'|})$, we have from \eqref{eqn:lemma.proof.union.bound.2} and \eqref{eqn:max.alpha.norm} that
\begin{align}\label{eqn:lemma.proof.union.bound.3}
&\mathbb{P}\left(\sup_{v \in \mathcal{S}_\alpha} \left\vert v^T \tilde{F}^r_k(w') -h^r_k(w^\prime)\right\vert \geq \epsilon^{\prime\prime}\right)\nonumber\\
&\qquad\qquad\quad\leq 2d_\xi \exp\left(-\frac{4|J'|N{\epsilon^{\prime\prime}}^2}{4C^2|J'|\|\bar{\alpha}\|^2 + {\epsilon^{\prime\prime}}^2}\right).
\end{align}

In order to obtain the desired uniform bound, we next need to remove the dependence on $\tilde{w}_k^r$ and $w'$ in \eqref{eqn:lemma.proof.union.bound.3}. To this end, we write $\tilde{F}^r_k(w')$ and $h^r_k(w^\prime)$ as $\tilde{F}^r_k(\tilde{w}_k^r, w')$ and $h^r_k(\tilde{w}_k^r,w')$, respectively, to highlight their dependence on $\tilde{w}_k^r$ and $w'$. Next, we define
\begin{align}
 \mathbb{U}_\zeta := \{u_1,\dots,u_{m_\zeta}\} \subset W \quad \text{s.t.} \quad W \subseteq \bigcup_{q=1}^{m_\zeta} B(u_q, \zeta)
\end{align}
to be a $\zeta$-covering of $W$ in terms of the $\ell_2$ norm. It then follows from \eqref{eqn:lemma.proof.union.bound.3} that
\begin{align}\label{eqn:lemma.proof.union.bound.4}
&\mathbb{P}\left(\sup_{v \in \mathcal{S}_\alpha, u \in \mathbb{U}_\zeta} \left\vert v^T \tilde{F}^r_k(u,[u]_k) -h^r_k(u,[u]_k)\right\vert \geq \epsilon^{\prime\prime}\right)\nonumber\\
&\qquad\qquad\quad\leq 2d_\xi m_\zeta \exp\left(-\frac{4|J'|N{\epsilon^{\prime\prime}}^2}{4C^2|J'|\|\bar{\alpha}\|^2 + {\epsilon^{\prime\prime}}^2}\right).
\end{align}
Similar to \eqref{eqn:lemma.proof.sup.bounds.1}, and using notation $\mathcal{L} := \{w' \in \mathbb{R}:|w'| \leq \Gamma\}$, we can also write
\begin{align}\label{eqn:lemma.proof.sup.bounds.2}
  &\sup_{v \in \mathcal{S}_\alpha, w \in W, w' \in \mathcal{L}} \left\vert v^T \tilde{F}^r_k(w,w') -h^r_k(w, w^\prime)\right\vert\nonumber\\
  &\quad\leq \sup_{u \in \mathbb{U}_\zeta, v \in \mathcal{S}_\alpha} \left|v^T \tilde{F}^r_k(u,[u]_k) -h^r_k(u,[u]_k)\right| +\nonumber\\
  &\qquad \sup_{u \in \mathbb{U}_\zeta, v \in \mathcal{S}_\alpha, w \in W, w' \in \mathcal{L}}\bigg[\left|v^T \tilde{F}^r_k(w,w')-v^T \tilde{F}^r_k(u,[u]_k)\right| + \nonumber\\
  &\qquad\qquad\qquad\qquad\qquad\left|h^r_k(u, [u]_k)-h^r_k(w, w^\prime)\right|\bigg].
\end{align}
Further, since $w\vert_{[w]_k=w'} \in W$ for any $(w,w') \in W \times \mathcal{L}$, we have from Assumption~\ref{assumption lipschitz} and definition of the set $\mathbb{U}_\zeta$ that
\begin{align}
  \sup_{u,v,w,w'} \left|v^T \tilde{F}^r_k(w,w')-v^T \tilde{F}^r_k(u,[u]_k)\right| &\leq L'\zeta, \quad \text{and}\label{eqn:lemma.proof.sup.bounds.3}\\
  \sup_{u,v,w,w'} \left|h^r_k(u, [u]_k)-h^r_k(w, w^\prime)\right| &\leq L'\zeta\label{eqn:lemma.proof.sup.bounds.4}.
\end{align}
We now fix $\epsilon''' > 0$, and define $\epsilon^{\prime\prime} := \epsilon'''/2$ and $\zeta := \epsilon'''/4L'$. We then obtain the following from \eqref{eqn:lemma.proof.union.bound.3}--\eqref{eqn:lemma.proof.sup.bounds.4}:
\begin{align}\label{eqn:lemma.proof.union.bound.5}
&\mathbb{P}\left(\sup_{v \in \mathcal{S}_\alpha, \{\tilde{w}_k^r\}, w'} \left\vert H^r_k(w^\prime) -h^r_k(w^\prime)\right\vert \geq \epsilon''' \right)\nonumber\\
&\qquad\qquad\quad\leq 2d_\xi m_\zeta \exp\left(-\frac{4|J'|N {\epsilon'''}^2}{16C^2|J'|\|\bar{\alpha}\|^2 + {\epsilon'''}^2}\right).
\end{align}

To conclude, let us define the event
\begin{align}\label{eqn:lemma.proof.event.A}
  \mathcal{A}_\epsilon := \left\{\sup_{v \in \mathcal{S}_\alpha, \{\tilde{w}_k^r\}, w'} \left\vert H^r_k(w^\prime) -h^r_k(w^\prime)\right\vert < \frac{\epsilon}{2}\right\}
\end{align}
for any $\epsilon > 0$. Conditioned on this event, we have
\begin{align}\label{eqn:lemma.proof.uniform.bound}
    &\forall (r,k), \ h^r_k(\widehat{w}) - h^r_k(w^\star) = \underbrace{h^r_k(\widehat{w}) - H^r_k(\widehat{w})}_{< \epsilon/2}\nonumber\\
     &\qquad + \underbrace{H^r_k(w^\star) - h^r_k(w^\star)}_{< \epsilon/2} + \underbrace{H^r_k(\widehat{w}) - H^r_k(w^\star)}_{\leq 0} < \epsilon.
\end{align}
Therefore, given any $\epsilon > 0$, we have from \eqref{eqn:lemma.proof.union.bound.5}--\eqref{eqn:lemma.proof.uniform.bound} that
\begin{align}\label{eqn:lemma.proof.final.bound.1}
&\mathbb{P}\left(\sup_{r,k} \left[h^r_k(\widehat{w}) - h^r_k(w^\star)\right] \geq \epsilon \right)\nonumber\\
&\qquad\qquad\quad\leq 2d_\xi m_\zeta \exp\left(-\frac{4|J'|N \epsilon^2}{64C^2|J'|\|\bar{\alpha}\|^2 + \epsilon^2}\right),
\end{align}
where $\xi = \frac{\epsilon}{8C|J'|}$ and $\zeta = \frac{\epsilon}{8L'}$. The proof now follows from \eqref{eqn:lemma.proof.final.bound.1} and the following facts about the covering numbers of the sets $\mathcal{S}_\alpha$ and $W$: (1) Since $\mathcal{S}_\alpha$ is a subset of $\Delta$, which can be circumscribed by a sphere in $\mathbb{R}^{|J'|-1}$ of radius $\sqrt{\frac{|J'|-1}{|J'|}} < 1$, we can upper bound $d_\xi$ by $\left(\frac{24C|J'|}{\epsilon}\right)^{|J'|}$~\cite{Verger-Gaugry2005covering}; and (2) Since $W \subset \mathbb{R}^P$ can be circumscribed by a sphere in $\mathbb{R}^P$ of radius $\Gamma \sqrt{P}$, we can upper bound $m_\zeta$ by $\left(\frac{24 L' \Gamma \sqrt{P}}{\epsilon}\right)^{P}$.\qed

\section{Proof of Theorem~\ref{coordinate descent theorem}}\label{proof.byrdie.convergence.caseI}
\noindent Let us begin with any coordinate descent iteration $r$ and dimension $k$, which starts with some $\tilde{w}_k^r \in \mathbb{R}^P$ (recall that $\tilde{w}_1^1 \equiv 0$). In order to facilitate the proof, we explicitly write $w^\star \in \mathbb{R}$, defined in \eqref{define wastastk}, and $\widehat{w} \in \mathbb{R}$, defined in \eqref{define whatk}, as ${w^\star}^r_k$ and $\widehat{w}^r_k$, respectively, to bring out their dependence on the indices $(r,k)$. We now define $\widehat{w}^r \in \mathbb{R}^P$ as $[\widehat{w}^r]_k := \widehat{w}^r_k$ and recall from Lemma~\ref{scalar convergence} and the subsequent definitions that $\forall r \in \mathbb{N}, \forall j \in J', \lim_{T \to \infty} w^{r,T}_j = \widehat{w}^r$. It therefore suffices to show that, in the limit of large $\bar{r}$, the statistical risk of $\widehat{w}^{\bar{r}}$ approaches the minimum statistical risk.

We now fix an arbitrary $\epsilon' \in (0,1)$ and note that
\begin{align}\label{eqn:thm.byrdie.prob.event}
  \sup_{r,k} \left[h^r_k(\widehat{w}^r_k) - h^r_k({w^\star}^r_k)\right] < \epsilon'
\end{align}
with high probability due to Theorem~\ref{lemma vector case high probability}. Note that invoking Theorem~\ref{lemma vector case high probability} requires the conditions $|\widehat{w}^r_k| \leq \Gamma$, $|{w^\star}^r_k| \leq \Gamma$, and $\{\tilde{w}^r_k\}_{r,k} \subset W$. We will return to these conditions in the latter part of the proof. Going forward, we condition on the event described by \eqref{eqn:thm.byrdie.prob.event} and notice that $\forall w' \in \mathbb{R}$, we have
\begin{align}\label{eqn:ByRDiE.I.th.proof.1}
&h_k^r({w^\star}^r_k) \stackrel{(a)}{\leq} h_k^r(w') \stackrel{(b)}{\leq} h_k^r(\widehat{w}_k^{r-1})]+\frac{L}{2}\vert w'-\widehat{w}_k^{r-1}\vert^2\nonumber\\
    &\qquad\qquad\qquad\qquad\quad \ +[\nabla h_k^r(\widehat{w}_k^{r-1})]_k(w'-\widehat{w}_k^{r-1}),
\end{align}
where ($a$) follows from the definition of ${w^\star}^r_k$ and ($b$) follows from Assumption~\ref{assumption lipschitz}. Plugging $w'=\widehat{w}_k^{r-1}-\frac{1}{L}[\nabla h_k^r(\widehat{w}_k^{r-1})]_k$ in \eqref{eqn:ByRDiE.I.th.proof.1}, we obtain
\begin{align}
    &h_k^r(\widehat{w}_k^{r-1}) - h_k^r({w^\star}^r_k) \geq \frac{1}{2L}[\nabla h_k^r(\widehat{w}_k^{r-1})]_k^2\label{eqn:ByRDiE.I.th.proof.2}\\
\Leftrightarrow \quad &h_k^r(\widehat{w}_k^{r-1}) - h_k^r(\widehat{w}_k^{r}) > \frac{1}{2L}[\nabla h_k^r(\widehat{w}_k^{r-1})]_k^2 - \epsilon' \label{eqn:ByRDiE.I.th.proof.3}
\end{align}
where \eqref{eqn:ByRDiE.I.th.proof.2} follows from \eqref{eqn:thm.byrdie.prob.event}. We have from \eqref{eqn:ByRDiE.I.th.proof.3} that $h_k^r(\widehat{w}_k^{r})$ is a strict monotonically decreasing function of $r$ for all $k$ as long as $[\nabla h_k^r(\widehat{w}_k^{r-1})]_k^2 \geq 2L \epsilon'$. It therefore follows that there exists some $r_0 \in \mathbb{N}$ such that
\begin{align}
  \forall k, \forall r \geq r_0, \ &[\nabla h_k^r(\widehat{w}_k^r)]_k^2 < 4L\epsilon'\label{eqn:ByRDiE.I.th.proof.4}\\
   \Rightarrow \forall r \geq r_0, \ &\|\nabla\mathbb{E}[f(\widehat{w}^r,(x,y)]\| < 2 \sqrt{P} L \epsilon'\label{eqn:ByRDiE.I.th.proof.5}.
\end{align}
In addition, convexity of $\mathbb{E}[f(\cdot,(x,y))]$ dictates
\begin{align}\label{eqn:ByRDiE.I.th.proof.6}
    &\mathbb{E}[f(w^\ast,(x,y))] \geq \mathbb{E}[f(\widehat{w}^r,(x,y))]\nonumber\\
        &\qquad\qquad\qquad\qquad\quad+\nabla\mathbb{E}[f(\widehat{w}^r,(x,y)]^T(w^\ast-\widehat{w}^r).
\end{align}
Using \eqref{eqn:ByRDiE.I.th.proof.5}, \eqref{eqn:ByRDiE.I.th.proof.6}, and the Cauchy--Schwarz inequality yields
\begin{align}\label{eqn:ByRDiE.I.th.proof.7}
  &\forall r \geq r_0, \ \mathbb{E}[f(\widehat{w}^r,(x,y))] - \mathbb{E}[f(w^\ast,(x,y))]\nonumber\\
  &\qquad\leq \|\nabla\mathbb{E}[f(\widehat{w}^r,(x,y)]\| \|\widehat{w}^r - w^\ast\| < 2 P L \Gamma \epsilon'.
\end{align}
Setting $\epsilon' = \epsilon/c_4$ in \eqref{eqn:ByRDiE.I.th.proof.6} and removing conditioning on \eqref{eqn:thm.byrdie.prob.event} using Theorem~\ref{lemma vector case high probability} give us the desired bound.

We conclude by commenting on the validity of $|\widehat{w}^r_k| \leq \Gamma$, $|{w^\star}^r_k| \leq \Gamma$, and $\{\tilde{w}^r_k\}_{r,k} \subset W$ needed for Theorem~\ref{lemma vector case high probability}. The condition holds for $\tilde{w}^1_1$ and ${w^\star}^1_1$ from the definition of the set $W$. The proof of Theorem~\ref{lemma vector case high probability} and a union bound argument also tells us that we can augment the bound in Theorem~\ref{lemma vector case high probability} with $\left[h^1_1(\widehat{w}^1_1) - h^1_1({w^\star}^1_1)\right] < \epsilon'$ without either requiring $|\widehat{w}^1_1| \leq \Gamma$ or exploding the probability of failure. This however leads to the condition $h_1^1(\widehat{w}_1^{1}) < h_1^1(0)$ due to \eqref{eqn:ByRDiE.I.th.proof.3}, which implies $|\widehat{w}_1^{1}| \leq \Gamma$, $\tilde{w}^1_2 \in W$, and $|{w^\star}^1_2| \leq \Gamma$. We can therefore revert to the original probability bound of Theorem~\ref{lemma vector case high probability} and start over the same argument with the knowledge that $\tilde{w}^1_2 \in W$, $|\widehat{w}_1^{1}| \leq \Gamma$, and $|{w^\star}^1_2| \leq \Gamma$. The rest of the claim follows by induction.\qed

\section{Proof of Theorem~\ref{thm:algorithmic.convergence.rate}}\label{proof.thm:algorithmic.convergence.rate}
\noindent We begin by defining the notation $\bar{f}(w) := \mathbb{E}[f(w,(x,y))]$ for $w \in \mathbb{R}^P$. Next, we borrow the notation of $\widehat{w}^r \in \mathbb{R}^P$ and some facts from the proof of Theorem~\ref{coordinate descent theorem} in Appendix~\ref{proof.byrdie.convergence.caseI}. These facts include the following. First, $\widehat{w}^r$ is the output of ByRDiE after each iteration $r$ at all nonfaulty nodes, i.e., $\forall r \in \mathbb{N}, \forall j \in J', \lim_{T \to \infty} w^{r,T}_j = \widehat{w}^r$. Second, defining $\epsilon' := \epsilon/c_4$, we have
\begin{align}\label{eqn:proof.thm.algo.rate.1}
  \forall (r,k), \ h_k^r(\widehat{w}_k^{r}) - h_k^r(\widehat{w}_k^{r-1}) < - \epsilon'
\end{align}
with probability $\geq 1 - \delta(\epsilon, N, \bar{a})$ as long as $\bar{f}(\widehat{w}^r) - \bar{f}^* > \epsilon$ (see, e.g., \eqref{eqn:ByRDiE.I.th.proof.3} and the discussion around it). Conditioning on the probability event described by~\eqref{eqn:proof.thm.algo.rate.1}, the definition of $h_k^r(\cdot)$ and \eqref{eqn:proof.thm.algo.rate.1} then give us the following recursion in $r$:
\begin{align}\label{eqn:proof.thm.algo.rate.2}
  \forall r, \ \bar{f}(\widehat{w}^r) - \bar{f}(\widehat{w}^{r-1}) < - \epsilon'
\end{align}
as long as $\bar{f}(\widehat{w}^r) - \bar{f}^* > \epsilon$. Note here that $\widehat{w}^0 = \tilde{w}^1_1 \equiv 0$. Therefore, \eqref{eqn:proof.thm.algo.rate.2} and the telescoping sum argument give us
\begin{align}\label{eqn:proof.thm.algo.rate.3}
  \bar{f}(\widehat{w}^r) < \bar{f}^0 - r\epsilon' \ \Leftrightarrow \ \bar{f}(\widehat{w}^r) - \bar{f}^* < \bar{f}^0 - \bar{f}^* - r\epsilon'&\\
  \Leftrightarrow \bar{f}(\widehat{w}^r) - \bar{f}^* < \left(\bar{f}^0 - \bar{f}^*\right)\left(1 - \frac{r \epsilon'}{\bar{f}^0 - \bar{f}^*}\right)&.
\end{align}
Plugging the inequality $\bar{f}^0 - \bar{f}^* \leq L'\|w^*\| \leq L'\sqrt{P}\gamma^*$ in \eqref{eqn:proof.thm.algo.rate.3} completes the proof.\qed

\section{Proof of Lemma~\ref{lemma Q converge}}\label{proof of lemma Q convergence}
\noindent Similar to the proof of Theorem~\ref{thm:algorithmic.convergence.rate}, we once again use the notation $\bar{f}(Q(q)) := \mathbb{E}[f(Q(q),(x,y))]$ to denote the statistical risk incurred by $Q(q)$ and show $\bar{f}(Q(q)) \xrightarrow{q,N} \bar{f}(w^*)$ in probability. It then follows from \cite[Theorem~4.4]{PlanidenWang.SJO16} and our assumptions that $w^*$ is a \emph{strong minimizer} of $\bar{f}(\cdot)$ and, therefore, $Q(q) \xrightarrow{q,N} w^*$ in probability.

In order to prove the aforementioned claim, we fix any $\epsilon > 0$ and show that $\bar{f}(Q(q)) - \bar{f}(w^*) \leq \epsilon$ for all large enough $q$ with probability that approaches $1$ as $N \rightarrow \infty$. To this end, we claim that $\bar{f}(Q(q))$ for all $q$ greater than some $q_0 \in \mathbb{N}$ is a strictly monotonically decreasing function with probability $1$, which is also lower bounded by $\bar{f}(w^*)$. By the monotone convergence theorem, therefore, $\bar{f}(Q(q))$ converges. We claim this convergence only takes place when $|[\nabla \bar{f}(Q(q))]_k| \leq \epsilon_\nabla$. Relegating the validity of this claim to the latter part of this proof, this means that $|[\nabla \bar{f}(Q(q))]_k|$ eventually becomes smaller than $\epsilon_\nabla$ with probability $1$ for large enough $q$, which implies
\begin{align}\label{eqn:lemma Q converge.1}
  \bar{f}(Q(q)) - \bar{f}(w^\ast) &\leq  - \nabla\bar{f}(Q(q))^T(w^\ast - Q(q))\nonumber\\
    &\leq \|\nabla\bar{f}(Q(q))\|_2 (\|w^*\|_2 + \|Q(q)\|_2)\nonumber\\
    &\leq (\sqrt{P} \epsilon_\nabla)\cdot(2\sqrt{P}\Gamma) < \epsilon
\end{align}
because of convexity of $\bar{f}(\cdot)$, the Cauchy--Schwarz inequality, and our assumptions. Since this is the desired result, we need now focus on the claim of strict monotonicity of $\bar{f}(Q(q))$ for this lemma. To prove this claim, note from Assumption~\ref{assumption lipschitz} that
\begin{align}\label{eqn:lemma Q converge.2}
    &\bar{f}(Q(q+1))\leq \bar{f}(Q(q)) + \nabla \bar{f}(Q(q))^T(Q(q+1)-Q(q))\nonumber\\
        &\qquad\qquad\qquad\qquad\quad+\frac{L}{2}\Vert Q(q+1)-Q(q)\Vert^2_2\nonumber\\
    &\qquad\qquad\stackrel{(a)}{=} \bar{f}(Q(q))+[\nabla \bar{f}(Q(q))]_k[(Q(q+1)-Q(q))]_k\nonumber\\
        &\qquad\qquad\qquad\qquad\quad+\frac{L}{2}\big|[Q(q+1)-Q(q)]_k\big|^2,
\end{align}
where (a) follows from the fact that $Q(q+1)-Q(q)$ is only nonzero in dimension $k$.

Next, we rewrite \eqref{update Q} as follows:
\begin{align}\label{eqn:update of Q}
    &Q(q+1)=Q(q)-\eta(q)\sum\limits_{i=1}^{\vert J'\vert}[\pi(r+1)]_i\Big([\nabla \widehat{f}(Q(q),S_i)]_k -\nonumber\\
    &\qquad\qquad\qquad\quad[\nabla \widehat{f}(Q(q),S_i)]_k + [\nabla \widehat{f}(w_i^r,S_i)]_k\Big)e_k\nonumber\\
    &\quad= Q(q) - \eta(q)\sum\limits_{i=1}^{\vert J'\vert}[\pi(r+1)]_i [\nabla \widehat{f}(Q(q),S_i)]_ke_k + E(q),
\end{align}
where $E(q) := \eta(q)\sum_{i=1}^{\vert J'\vert}[\pi(r+1)]_i\big([\nabla \widehat{f}(Q(q),S_i)]_k-[\nabla \widehat{f}(w_i^r,S_i)]_k\big)e_k$. Plugging this into \eqref{eqn:lemma Q converge.2} results in
\begin{align}\label{eqn:lemma Q converge.3}
    &\bar{f}(Q(q))-\bar{f}(Q(q+1)) \geq -[E(q)]_k[\nabla\bar{f}(Q(q))]_k\nonumber\\
    &+\eta(q)\sum\limits_{i=1}^{\vert J'\vert}[\pi(r+1)]_i [\nabla \widehat{f}(Q(q),S_i)]_k[\nabla\bar{f}(Q(q))]_k\nonumber\\
    &-\frac{L}{2}\Big\vert [E(q)]_k- \eta(q)\sum\limits_{i=1}^{\vert J'\vert}[\pi(r+1)]_i [\nabla \widehat{f}(Q(q),S_i)]_k\Big\vert^2.
\end{align}
The right-hand side of \eqref{eqn:lemma Q converge.3} strictly lower bounded by $0$ implies strict monotonicity of $\bar{f}(Q(q))$. Simple algebraic manipulations show that this is equivalent to the condition
\begin{align}\label{eqn:lemma Q converge.condition}
    &\frac{L\eta(q)^2}{2}\Big(\sum\limits_{i=1}^{\vert J'\vert}[\pi(r+1)]_i [\nabla \widehat{f}(Q(q),S_i)]_k\Big)^2\nonumber\\
    &\qquad< \eta(q)\sum\limits_{i=1}^{\vert J'\vert}[\pi(r+1)]_i [\nabla\widehat{f}(Q(q),S_i)]_k[\nabla\bar{f}(Q(q))]_k\nonumber\\
    &\qquad\qquad+ L\eta(q)[E(q)]_k \sum\limits_{i=1}^{\vert J'\vert}[\pi(r+1)]_i [\nabla \widehat{f}(Q(q),S_i)]_k\nonumber\\
    &\qquad\qquad\qquad-[E(q)]_k[\nabla\bar{f}(Q(q))]_k-\frac{L}{2}[E(q)]_k^2.
\end{align}

Next, notice $\mathbb{E}[\sum_{i=1}^{\vert J'\vert}[\pi(r+1)]_i[\nabla \widehat{f}(Q(q),S_i)]_k]=\mathbb{E}[\nabla f(Q(q),(x,y))]_k$. We now make an assumption whose validity is also discussed at the end of the proof. We assume $\exists q_0' \in \mathbb{N}:\forall q \geq q_0', Q(q)\in W$, in which case we can show using arguments similar to the ones in the proof of Theorem~\ref{lemma vector case high probability} (cf.~Appendix~\ref{proof of lemma vector case high probability}) that
\begin{align}\label{eqn:proof.probability.statement}
\mathbb{P}\big(\vert \sum_{i \in J'}[\pi(r+1)]_i[\nabla \widehat{f}(Q(q),S_i)]_k - [\nabla \bar{f}(Q(q))]_k \vert \leq \epsilon'\big)
\end{align}
converges to $1$ uniformly for all $(r,k)$ (equivalently, all $q$) for any $\epsilon' > 0$. We therefore have with probability $1$ (as $N \rightarrow \infty$)
\begin{align}\label{epsilon square}
    &\Big(\sum\limits_{i=1}^{\vert J'\vert}[\pi(r+1)]_i [\nabla \widehat{f}(Q(q),S_i)]_k\Big)^2 \leq [\nabla\bar{f}(Q(q))]_k^2 + \epsilon'^2\nonumber\\
    &\qquad\qquad\qquad\qquad\qquad\qquad\quad+ 2\epsilon'\Big\vert[\nabla\bar{f}(Q(q))]_k\Big\vert.
\end{align}
Next, we consider two cases: ($i$) $[\nabla\bar{f}(Q(q))]_k > 0$, and ($ii$) $[\nabla\bar{f}(Q(q))]_k < 0$. When $[\nabla\bar{f}(Q(q))]_k > 0$, we have in probability $\sum_{i \in J'}[\pi(r+1)]_i [\nabla \widehat{f}(Q(q),S_i)]_k \geq [\nabla\bar{f}(Q(q))]_k -\epsilon'$. This fact along with \eqref{epsilon square}, the realization that $\frac{L\eta(q)^2}{2}-\eta(q)<0$ for large enough $q$ since $\eta(q) \rightarrow 0$, and some tedious but straightforward algebraic manipulations show that the following condition is sufficient for \eqref{eqn:lemma Q converge.condition} to hold in probability:
\begin{align}\label{condition combined}
    &\Big\vert[\nabla\bar{f}(Q(q))]_k\Big\vert >\nonumber\\
    & \bigg(\frac{2\epsilon'\big\vert[\nabla\bar{f}(Q(q))]_k\big\vert}{2-{L\eta(q)}} + \frac{2\big\vert[E(q)]_k\big\vert\big\vert[\nabla\bar{f}(Q(q))]_k\big\vert}{2\eta(q)-L\eta(q)^2}\nonumber\\
    &+ \frac{2L\big\vert[E(q)]_k\big\vert \big\vert\sum_{i\in J'}[\pi(r+1)]_i [\nabla \widehat{f}(Q(q),S_i)]_k\big\vert}{2-L\eta(q)}\nonumber\\
    &+ \frac{L\big\vert[E(q)]_k\big\vert^2}{2\eta(q)-L\eta(q)^2} + \frac{2L\eta(q)\epsilon'\big\vert [\nabla\bar{f}(Q(q))]_k \big\vert}{2-L\eta(q)} + \frac{L\eta(q)\epsilon'^2}{2-L\eta(q)}\bigg)^\frac{1}{2}.
\end{align}

Using similar arguments, one can also show that the case $[\nabla\bar{f}(Q(q))]_k < 0$ also results in \eqref{condition combined} as a sufficient condition for \eqref{eqn:lemma Q converge.condition} to hold in probability. We now note that $|[\nabla \bar{f}(\cdot)]_k|$ and $|[\nabla \widehat{f}(\cdot,\cdot)]_k|$ in \eqref{condition combined} can be upper bounded by some constants $L_{\bar{\nabla}}$ and $L_\nabla$ by virtue of Assumption~\ref{assumption lipschitz} and the definition of $W$. This results in the following sufficient condition for \eqref{eqn:lemma Q converge.condition}:
\begin{align}\label{condition combined.2}
    &\Big\vert[\nabla\bar{f}(Q(q))]_k\Big\vert > \bigg(\frac{2L_{\bar{\nabla}}\epsilon' + 2LL_{\bar{\nabla}}\eta(q)\epsilon' + L\eta(q)\epsilon'^2}{2-{L\eta(q)}}\nonumber\\
    &+\frac{2LL_\nabla\big\vert[E(q)]_k\big\vert }{2-L\eta(q)} + \frac{2L_{\bar{\nabla}}\big\vert[E(q)]_k\big\vert + L\big\vert[E(q)]_k\big\vert^2}{2\eta(q)-L\eta(q)^2}\bigg)^\frac{1}{2}.
\end{align}
The right-hand side of \eqref{condition combined.2} can be made arbitrarily small (and, in particular, equal to $\epsilon_\nabla$) through appropriate choice of $\epsilon'$ and large enough $q$; indeed, we have from our assumptions, Theorem~\ref{lemma consensus}, and the definitions of $\eta(q)$ and $E(q)$ that both $[E(q)]_k$ and $[E(q)]_k/\eta(q)$ converge to $0$ as $q \rightarrow \infty$.

This completes the proof, except that we need to validate one remaining claim and discuss one assumption. The claim is that $\bar{f}(Q(q))$ cannot converge when $\vert[\nabla\bar{f}(Q(q))]_k\vert >\epsilon_\nabla$. We prove this by contradiction. Suppose $\exists k \in \{1,\dots,P\}$ and $\epsilon_0 > 0$ such that $\vert[\nabla\bar{f}(Q(q))]_k\vert - \epsilon_\nabla > \epsilon_0$ for all $q$. We know $\exists q_0 \in \mathbb{N}$ such that the right hand side of \eqref{condition combined.2} becomes smaller than $\epsilon_\nabla$ for all $q \geq q_0$. Therefore, adding $\epsilon_0$ to the right hand side of \eqref{condition combined.2} and combining with \eqref{eqn:lemma Q converge.3} gives $\forall q \geq q_0$:
\begin{align}\label{eqn: contradiction}
\bar{f}(Q(q))-\bar{f}(Q(q+1)) \geq (2\eta(q)-L\eta(q)^2)(\epsilon_0^2+2\epsilon_\nabla\epsilon_0).
\end{align}
Taking summation on both sides of \eqref{eqn: contradiction} from $q=q_0$ to $\infty$, and noting that $\sum_{q=q_0}^{\infty} \eta(q) = \infty$ and $\sum_{q=q_0}^{\infty} \eta(q)^2 < \infty$, gives us $\bar{f}(Q(q_0)) - \lim_{q \to \infty} \bar{f}(Q(q)) = \infty$. This contradicts the fact that $\bar{f}(\cdot)$ is lower bounded, thereby validating our claim.

Finally, the assumption $\exists q_0' \in \mathbb{N}:\forall q \geq q_0', Q(q)\in W$ is true with probability $1$ (as $N\to\infty$) by virtue of the facts that $W$ is defined in terms of the sublevel set of $\bar{f}(\cdot)$, \eqref{eqn:proof.probability.statement} holds $\forall q < q_0'$ without requiring the assumption, $\exists q_0 \in \mathbb{N}$ such that $\bar{f}(Q(q))$ is monotonic in $q$ for all $q \geq q_0$ due to \eqref{eqn:proof.probability.statement}, and the probabilistic ``onion peeling'' induction argument at the end of the proof of Theorem~\ref{coordinate descent theorem} (cf.~Appendix~\ref{proof.byrdie.convergence.caseI}) is applicable in this case also (except that one will have to start the argument from some index $q = q_0' \geq q_0$).\qed

\balance

\end{document}